\documentclass[10pt]{article}
\usepackage{authblk}
\usepackage[margin=1in]{geometry} 
\usepackage{amsmath,amsthm,amssymb}
\usepackage{enumitem}
\usepackage{soul}
\usepackage{mathtools}
\usepackage[utf8]{inputenc}
\usepackage{multirow}
\usepackage[colorlinks]{hyperref}
\usepackage{cleveref}
\usepackage{bbm}
\usepackage{tikz-cd}
\usepackage{adjustbox}
\usepackage{algorithm, algpseudocode}
\allowdisplaybreaks

\makeatletter
\DeclareFontFamily{OMX}{MnSymbolE}{}
\DeclareSymbolFont{MnLargeSymbols}{OMX}{MnSymbolE}{m}{n}
\SetSymbolFont{MnLargeSymbols}{bold}{OMX}{MnSymbolE}{b}{n}
\DeclareFontShape{OMX}{MnSymbolE}{m}{n}{
    <-6>  MnSymbolE5
    <6-7>  MnSymbolE6
    <7-8>  MnSymbolE7
    <8-9>  MnSymbolE8
    <9-10> MnSymbolE9
    <10-12> MnSymbolE10
    <12->   MnSymbolE12
}{}
\DeclareFontShape{OMX}{MnSymbolE}{b}{n}{
    <-6>  MnSymbolE-Bold5
    <6-7>  MnSymbolE-Bold6
    <7-8>  MnSymbolE-Bold7
    <8-9>  MnSymbolE-Bold8
    <9-10> MnSymbolE-Bold9
    <10-12> MnSymbolE-Bold10
    <12->   MnSymbolE-Bold12
}{}

\numberwithin{equation}{section}

\newcommand{\indep}{\perp \!\!\! \perp}

\newcommand{\bftheta}{\boldsymbol{\theta}}

\newcommand{\bfepsilon}{\boldsymbol{\epsilon}}

\newcommand{\bfa}{\mathbf{a}}
\newcommand{\bfA}{\mathbf{A}}

\newcommand{\bfm}{\mathbf{m}}

\newcommand{\bfo}{\mathbf{o}}
\newcommand{\bfO}{\mathbf{O}}

\newcommand{\bfs}{\mathbf{s}}

\newcommand{\bfu}{\mathbf{u}}

\newcommand{\bfv}{\mathbf{v}}

\newcommand{\bfx}{\mathbf{x}}
\newcommand{\bfX}{\mathbf{X}}
\newcommand{\bfy}{\mathbf{y}}
\newcommand{\bfY}{\mathbf{Y}}
\newcommand{\bfz}{\mathbf{z}}

\newcommand{\bfzero}{\mathbf{0}}

\newcommand{\bbA}{\mathbb{A}}
\newcommand{\bbB}{\mathbb{B}}

\newcommand{\bbE}{\mathbb{E}}

\newcommand{\bbI}{\mathbb{I}}

\newcommand{\bbL}{\mathbb{L}}

\newcommand{\bbN}{\mathbb{N}}
\newcommand{\bbO}{\mathbb{O}}
\newcommand{\bbP}{\mathbb{P}}

\newcommand{\bbR}{\mathbb{R}}
\newcommand{\bbS}{\mathbb{S}}

\newcommand{\bbX}{\mathbb{X}}
\newcommand{\bbY}{\mathbb{Y}}

\newcommand{\calA}{\mathcal{A}}

\newcommand{\calC}{\mathcal{C}}
\newcommand{\calD}{\mathcal{D}}

\newcommand{\calF}{\mathcal{F}}
\newcommand{\calG}{\mathcal{G}}
\newcommand{\calH}{\mathcal{H}}
\newcommand{\calI}{\mathcal{I}}

\newcommand{\calL}{\mathcal{L}}
\newcommand{\calM}{\mathcal{M}}
\newcommand{\calN}{\mathcal{N}}

\newcommand{\calR}{\mathcal{R}}

\newcommand{\calV}{\mathcal{V}}

\newcommand{\calZ}{\mathcal{Z}}

\newtheorem{defi}{Definition}[section]
\newtheorem{thm}{Theorem}[section]
\newtheorem{cor}{Corollary}[section]
\newtheorem{prop}{Proposition}[section]
\newtheorem{lemma}{Lemma}[section]

\newtheorem{assumption}{Assumption}

\newcommand{\RNum}[1]{\uppercase\expandafter{\romannumeral #1\relax}}

\DeclareMathOperator*{\argmin}{argmin}

\let\llangle\@undefined
\let\rrangle\@undefined
\DeclareMathDelimiter{\llangle}{\mathopen}%
                     {MnLargeSymbols}{'164}{MnLargeSymbols}{'164}
\DeclareMathDelimiter{\rrangle}{\mathclose}%
                     {MnLargeSymbols}{'171}{MnLargeSymbols}{'171}
\makeatother

\begin{document}
 
\title{Off-policy estimation with adaptively collected data:\\ the power of online learning}

\author{Jeonghwan Lee\footnote{E-mail: \href{mailto:jhlee97@uchicago.edu}{\texttt{jhlee97@uchicago.edu}}.} \ and Cong Ma\footnote{E-mail: \href{mailto:congm@uchicago.edu}{\texttt{congm@uchicago.edu}}.}}
\affil{Department of Statistics,  University of Chicago}

\maketitle

\begin{abstract}
We consider estimation of a linear functional of the treatment effect using adaptively collected data. This task finds a variety of applications including the off-policy evaluation (\textsf{OPE}) in contextual bandits, and estimation of the average treatment effect (\textsf{ATE}) in causal inference. While a certain class of augmented inverse propensity weighting (\textsf{AIPW}) estimators enjoys desirable asymptotic properties including the semi-parametric efficiency, much less is known about their non-asymptotic theory with adaptively collected data. To fill in the gap, we first establish generic upper bounds on the mean-squared error of the class of AIPW estimators that crucially depends on a sequentially weighted error between the treatment effect and its estimates. Motivated by this, we also propose a general reduction scheme that allows one to produce a sequence of estimates for the treatment effect via online learning to minimize the sequentially weighted estimation error. To illustrate this, we provide three concrete instantiations in (\romannumeral 1) the tabular case; (\romannumeral 2) the case of linear function approximation; and (\romannumeral 3) the case of general function approximation for the outcome model. We then provide a local minimax lower bound to show the instance-dependent optimality of the \textsf{AIPW} estimator using no-regret online learning algorithms.
\end{abstract}

\section{Introduction}
\label{sec:intro}

Estimating a linear functional of the treatment effect is of great importance in both the literature of causal inference and reinforcement learning (\textsf{RL}). For instance, in causal inference, one is interested in estimating the average treatment effect (\textsf{ATE}) \cite{hirano2003efficient} or their weighted variants, and in the bandits and \textsf{RL} literature, one is interested in estimating the expected reward of a target policy \cite{li2015toward, wang2017optimal, ma2022minimax,li2023sharp}. Two main challenges arise when tackling this problem:
\begin{itemize}
    \item \textbf{Off-policy estimation}: Oftentimes, one needs to estimate a linear functional of the treatment effect based on observational dataset collected from a behavior policy. This behavior policy may not match the desired distribution specified by the linear functional \cite{mou2022off};
    \item \textbf{Adaptive data collection mechanism}: It is increasingly common for observational datasets to be adaptively collected due to the use of online algorithms (e.g., contextual bandit algorithms \cite{thompson1933likelihood, lai1985asymptotically, agrawal2013thompson, russo2018tutorial, lattimore2020bandit}) in experimental design \cite{zhan2021off}. 
\end{itemize}

\indent In this paper, we deal with two challenges simultaneously by investigating the estimation task of a linear functional of the treatment effect from an adaptively collected dataset. When the observational dataset is collected non-adaptively, i.e., in an i.i.d. manner, there is an extensive line of work \cite{robins1995analysis, robins1995semiparametric, dudik2011doubly, imbens2015causal, agarwal2017effective, kallus2021optimal, narita2019efficient, chernozhukov2018double, armstrong2021finite, wang2017optimal, ma2022minimax} investigating the asymptotic and non-asymptotic theory of a variety of estimators. Most notably are the study \cite{chernozhukov2018double} that establishes the asymptotic efficiency of a family of semi-parametric estimators, and a more recent work \cite{mou2022off} that undertakes a finite-sample analysis unveiling the importance of a certain weighted $\ell_2$-norm for estimation of a linear functional of the treatment effect. On the other hand, when it comes to adaptively collected data, most prior works \cite{hadad2019confidence, zhan2021off} focus on the asymptotic normality of the estimators, and do not discuss the finite-sample analysis of the estimators. In this paper, we aim to fill in this gap.

\subsection{Main contributions}
\label{subsec:main_contributions}
More specifically, we make the following three main contributions in this paper: 
\begin{itemize}
    \item First, we present generic finite-sample upper bounds on the mean-squared error (\textsf{MSE}) of the class of \emph{augmented inverse propensity weighting} (\textsf{AIPW}) estimators that crucially depends on a sequentially weighted error between the treatment effect and its estimates. This sequentially weighted estimation error demonstrates a clear effect of history-dependent behavior policies;
    \item Second, motivated by previous observations, we propose a general reduction scheme that allows one to form a sequence of estimates for the treatment effect via online learning to minimize the sequentially weighted estimation error. In order to demonstrate this, we provide three concrete instantiations in (\romannumeral 1) the tabular case; (\romannumeral 2) the case of linear function approximation; and (\romannumeral 3) the case of general function approximation for the outcome model;
    \item In the end, we provide a local minimax lower bound to showcase the instance-dependent optimality of the \textsf{AIPW} estimator using no-regret online learning algorithms in the large-sample regime.
\end{itemize}

\subsection{Related works}
\label{subsec:related_works}

\paragraph{Off-policy estimation with observational data}
Off-policy estimation in observational settings has been a central topic in statistics, operations research, causal inference, and \textsf{RL}. Here, we group a few prominent off-policy estimators into the following three categories: (\romannumeral 1) \emph{Model-based estimator}: it is often dubbed as the \emph{direct method} (\textsf{DM}), whose key idea is to utilize observational data to learn a regression model that predicts outcomes for each state-action pair, and then average these model predictions \cite{kang2007demystifying, dudik2011doubly, dudik2014doubly, little2019statistical}. Due to model mis-specification, the \textsf{DM} typically has a low variance but might lead to highly biased estimation results. (\romannumeral 2) \emph{Inverse propensity weighting} (\textsf{IPW}): for the \textsf{OPE} task, \textsf{IPW} uses importance weighting to account for the distribution mismatch between the behavioral policy and the target policy \cite{horvitz1952generalization, strehl2010learning}. If the behavioral policy differs significantly from the target policy, then the \textsf{IPW} can have an overly large variance (known as the \emph{low overlap} issue) \cite{imbens2004nonparametric}. Typical remedies for this issue include propensity clipping \cite{ionides2008truncated, su2019cab} or self-normalization \cite{hesterberg1995weighted, swaminathan2015self}. (\romannumeral 3) \emph{Hybrid estimator}: some estimators (e.g., the doubly-robust (\textsf{DR}) estimator \cite{dudik2011doubly}) combine \textsf{DM} and \textsf{IPW} together to leverage their complementary strengths \cite{robins2007comment, dudik2011doubly, dudik2014doubly, thomas2016data, farajtabar2018more, su2020doubly, wang2017optimal}. A central asymptotic results in \textsf{OPE} is that the cross-fitted \textsf{DR} estimator is $\sqrt{n}$-consistent and asymptotically efficient (that is, it attains the lowest possible asymptotic variance), even for the case where nuisance parameters are estimated at rates slower than $\sqrt{n}$-rates \cite{chernozhukov2018double}. However, these methods still might be vulnerable to the low overlap issue especially for large or continuous action spaces. Hence, there has been a line of recent explorations on \textsf{OPE} for large action spaces \cite{felicioni2022off, saito2022off, peng2023offline, saito2023off} and \textsf{OPE} for continuous action space \cite{kallus2018policy, lee2022local, wang2024oracle}.

\paragraph{Off-policy estimation with adaptively collected data}
A recent strand of works studied asymptotic theory of adaptive variants of the \textsf{IPW} and the \textsf{DR} estimators (e.g., asymptotic normality, semi-parametric efficiency, and confidence intervals) \cite{kato2020efficient, dai2024clip, cook2024semiparametric} for adaptively collected data. However, in adaptive experiments, overlap between the behavioral policies and the target policy can deteriorate since the experimenter shifts the behavioral policies in response to what the learner observes (well-known as the \emph{drifting overlap}) \cite{zhan2021off}. It may engender unacceptably large variances of the \textsf{IPW} and \textsf{DR} estimators. To address this large variance problem, there has been a recent strand of studies investigating variance reduction strategies for the \textsf{DR} estimator based on shrinking importance weights toward one \cite{bottou2013counterfactual, wang2017optimal, su2019cab, su2020doubly}, local stabilization \cite{luedtke2016statistical, zhang2020inference}, and adaptive weighting \cite{hadad2021confidence, zhan2021off}. Recent studies on policy learning with adaptively collected data \cite{zhan2023policy, jin2022policy} explored the adaptive weighting \textsf{DR} estimator for policy learning. In contrast with the majority of existing works on off-policy estimation using adaptively collected data that focus on asymptotic results, this paper aims at establishing non-asymptotic theory of the problem. While several researchers have been recently explored non-asymptotic results of the problem with an emphasis on uncertainty quantification \cite{karampatziakis2021off, waudby2024anytime}, we focus on analyses of estimation procedures of the off-policy value. As most existing standard objects for uncertainty quantification, such as a confidence interval (\textsf{CI}), take a very static view of the world (e.g., it holds for a fixed sample size and is not designed for interactive/adaptive data collection procedures), the aforementioned two papers \cite{karampatziakis2021off, waudby2024anytime} instead study a more suitable statistical tool for such cases called a \emph{confidence sequence}.

\section{Problem formulation}
\label{sec:problem_formulation}

We formulate our problem with the language of contextual bandits: let $\bbX$, $\bbA$, and $\bbY \subseteq \bbR$ denote the \emph{context space}, the \emph{action space}, and the \emph{outcome space}, respectively. Denote by $\bbO := \bbX \times \bbA \times \bbY$ the space of all possible context-action-outcome triples. In an adaptive experiment, we observe $n$ samples $\left\{ \left( X_i, A_i, Y_i \right) \in \bbO : i \in [n] \right\}$ produced by the following data generating procedure \cite{jin2022policy, zhan2023policy}: At each stage $i \in [n]$, 
\begin{enumerate} [label = (\roman*)]
    \item A context $X_i \in \bbX$ is independently sampled from a fixed \emph{context distribution} $\Xi^* (\cdot) \in \Delta (\bbX)$;
    \item There exists a \emph{behavioral policy} $\Pi_{i}^* (\cdot, \cdot): \bbX \times \bbO^{i-1} \to \Delta (\bbA)$ that selects the $i$-th action as $A_i \left| X_i, \bfO_{i-1} \right. \sim \Pi_{i}^* \left( \cdot \left| X_i, \bfO_{i-1} \right. \right)$, where $\bfO_i := \left( X_1, A_1, Y_1, \cdots, X_i, A_i, Y_i \right) \in \bbO^i$ for $i \in [n]$. Since $\Pi_{i}^* \left( \cdot \left| X_i, \bfO_{i-1} \right. \right)$ may depend on previous observations, $\left\{ \left( X_i, A_i, Y_i \right): i \in [n] \right\}$ are no longer i.i.d.;
    \item Given a Markov kernel $\Gamma^* (\cdot, \cdot): \bbX \times \bbA \to \Delta (\bbY)$, we assume that the outcome is generated according to $Y_i \sim \Gamma^* \left( \cdot \left| X_i, A_i \right. \right)$. Moreover, the conditional mean of the outcome $Y_i \in \bbY$ is specified as
    \[
        \bbE \left[ Y_i \left| X_i, A_i \right. \right] = \int_{\bbY} y \Gamma^* \left( \mathrm{d} y \left| X_i, A_i \right. \right) = \mu^* \left( X_i, A_i \right),
    \]
    where the function $\mu^* (\cdot, \cdot): \bbX \times \bbA \to \bbR$ is called the \emph{treatment effect} (in the literature of causal inference) or the \emph{reward function} (in bandit and \textsf{RL} literature). We note that the treatment effect $\mu^*$ is not revealed to the statistician. We also define the conditional variance function $\sigma^2 (\cdot, \cdot): \bbX \times \bbA \to \left[ 0, +\infty \right]$ defined by $\sigma^2 \left( x, a \right) := \bbE \left[ \left. \left\{ Y - \mu^* \left( X, A \right) \right\}^2 \right| \left( X, A \right) = \left( x, a \right) \right]$, which is assumed to satisfy $\sigma^2 (x, a) < +\infty$ for all state-action pairs $(x, a) \in \bbX \times \bbA$.
\end{enumerate}

\indent We embark on our discussion after assuming the existence of $\sigma$-finite base measures $\lambda_{\bbX} (\cdot)$, $\lambda_{\bbA} (\cdot)$, and $\lambda_{\bbY} (\cdot)$ over $\bbX$, $\bbA$, and $\bbY$, respectively, such that $\Xi^* (\cdot) \ll \lambda_{\bbX} (\cdot)$, $\Pi_{i}^* \left( \cdot \left| x, \bfo_{i-1} \right. \right) \ll \lambda_{\bbA} (\cdot)$ for every $\left( x, \bfo_{i-1} \right) \in \bbX \times \bbO^{i-1}$ and $i \in [n]$, and $\Gamma^* \left( \cdot \left| x, a \right. \right) \ll \lambda_{\bbY} (\cdot)$ for all state-action pairs $(x, a) \in \bbX \times \bbA$. The notation $\ll$ stands for the \emph{absolute continuity} between measures. Our main objective is to estimate the \emph{off-policy value} for any given target evaluation function $g (\cdot, \cdot): \bbX \times \bbA \to \bbR$ defined as
\begin{equation}
    \label{eqn:off_policy_value_functional}
    \begin{split}
        \tau^* = \tau \left( \calI^* \right) := \bbE_{X \sim \Xi^*} \left[ \left\langle g (X, \cdot), \mu^* (X, \cdot) \right\rangle_{\lambda_{\bbA}} \right],
    \end{split}
\end{equation}
where $\calI^* := \left( \Xi^*, \Gamma^* \right) \in \bbI := \Delta (\bbX) \times \left( \bbX \times \bbA \to \Delta (\bbY) \right)$ refers to the ground-truth \emph{problem instance}, and given any pair of functions $(f, g) \in \left( \bbA \to \bbR \right) \times \left( \bbA \to \bbR \right)$ such that $fg \in \bbL^1 \left( \bbA, \lambda_{\bbA} \right)$, we define the inner product
\[
    \left\langle f, g \right\rangle_{\lambda_{\bbA}} := \int_{\bbA} f (a) g (a) \mathrm{d} \lambda_{\bbA} (a).
\]
In this paper, we assume that the propensity scores $\left\{ \pi_{i}^* \left( X_i, \bfO_{i-1}; A_i \right) : i \in [n] \right\}$ are known to the statistician, where $\pi_{i}^* \left( x, \bfo_{i-1}; \cdot \right) := \frac{\mathrm{d} \Pi_{i}^* \left( \cdot \left| x, \bfo_{i-1} \right. \right)}{\mathrm{d} \lambda_{\bbA}}: \bbA \to \bbR_{+}$.
\medskip

\indent As we mentioned earlier in Section \ref{sec:intro}, the estimation problem of a linear functional of the treatment effect $\mu^*$ turns out to be useful in both causal inference and \textsf{RL} literature in the following sense:

\begin{itemize}[leftmargin=10pt]
    \item \textbf{Estimation of average treatment effects}: Consider the binary action space $\bbA = \left\{ 0, 1 \right\}$ equipped with the counting measure over $\bbA$. The \emph{average treatment effect} (\textsf{ATE}) in our problem setting is defined as the linear functional
    \[
        \begin{split}
            \textsf{ATE} := \bbE_{\calI^*} \left[ Y_{i} (1) - Y_{i} (0) \right] = \bbE_{X \sim \Xi^*} \left[ \mu^* \left( X, 1 \right) - \mu^* \left( X, 0 \right) \right].
        \end{split}
    \]
    Once we take the target evaluation function as $g (x, a) = 2a - 1$, the \textsf{ATE} boils down to a particular case of the equation \eqref{eqn:off_policy_value_functional};
    \item \textbf{Off-policy evaluation for contextual bandits}: We assume that a \emph{target policy} $\Pi^{\textnormal{target}} (\cdot): \bbX \to \Delta (\bbA)$ is given such that $\Pi^{\textnormal{target}} \left( \cdot \left| x \right. \right) \ll \lambda_{\bbA} (\cdot)$ for every context $x \in \bbX$. For simplicity, we denote $\pi^{\textnormal{target}} \left( x, \cdot \right) := \frac{\mathrm{d} \Pi^{\textnormal{target}} \left( \cdot \left| x \right. \right)}{\mathrm{d} \lambda_{\bbA}}: \bbA \to \bbR_{+}$ by the density function of the target policy for each context $x \in \bbX$. Once we take $g (x, a) = \pi^{\textnormal{target}} (x, a)$, then the linear functional \eqref{eqn:off_policy_value_functional} corresponds to the value of the target policy $\Pi^{\textnormal{target}}$. This task has been widely studied in the literature of bandits and \textsf{RL}, known as the \emph{off-policy evaluation} (\textsf{OPE}).
\end{itemize}

\noindent We conclude this section by introducing notations that will be useful in later sections: let $\bbP_{\calI}^{i} (\cdot) \in \Delta \left( \bbO^{i} \right)$ denote the law of the sample trajectory $\bfO_{i}$ under the sampling mechanism with a problem instance $\calI = \left( \Xi, \Gamma \right) \in \bbI$. We denote the density function of $\bbP_{\calI}^{i} (\cdot) \in \Delta \left( \bbO^i \right)$ with respect to the base measure $\left( \lambda_{\bbX} \otimes \lambda_{\bbA} \otimes \lambda_{\bbY} \right)^{\otimes i}$ by $p_{\calI}^{i} (\cdot): \bbO^i \to \bbR_{+}$. Lastly, we define the \emph{$k$-th weighted $\ell_2$-norm} for $k \in [n]$ as
\begin{equation}
    \label{eqn:weighed_l2_norm}
    \left\| \varphi \right\|_{(k)}^2 := \frac{1}{k} \sum_{i=1}^{k} \bbE_{\calI^*} \left[ \frac{g^2 \left( X_i, A_i \right) \varphi^2 \left( X_i, A_i \right)}{\left( \pi_{i}^* \right)^2 \left( X_i, \bfO_{i-1}; A_i \right)} \right]
\end{equation}
for any function $\varphi (\cdot, \cdot): \bbX \times \bbA \to \bbR$, together with the \emph{$k$-th weighted $\ell_2$-space} by
\[
    \bbL_{(k)}^2 := \left\{ \varphi (\cdot, \cdot) \in \left( \bbX \times \bbA \to \bbR \right): \left\| \varphi \right\|_{(k)} < +\infty \right\}.
\]

\section{A class of \textsf{AIPW} estimators and its non-asymptotic guarantees}
\label{sec:AIPW_non_asymptotic_guarantees}

The main objective of this section is to develop a \emph{meta-algorithm} to tackle the estimation problem of the off-policy value \eqref{eqn:off_policy_value_functional}, followed by a key rationale of the proposed procedure as a variance-reduction scheme of the standard \emph{inverse propensity weighting} (\textsf{IPW}) estimator.

\subsection{How can we reduce the variance of the \textsf{IPW} estimator?}
\label{subsubsec:variance_reduction}

Similar to the prior work \cite{mou2022off}, we consider a class of two-stage estimators obtained from simple perturbations of the \textsf{IPW} estimator. Given any collection $f := \left( f_i : \bbX \times \bbO^{i-1} \times \bbA \to \bbR : i \in [n] \right)$ of auxiliary functions, we consider the following \emph{perturbed} \textsf{IPW} \emph{estimator} $\hat{\tau}_{n}^{f} \left( \cdot \right): \bbO^n \to \bbR$:
\[
    \hat{\tau}_{n}^{f} \left( \bfo_n \right) := \frac{1}{n} \sum_{i=1}^{n} \left\{ \frac{g \left( x_i, a_i \right) y_i}{\pi_{i}^* \left( x_i, \bfo_{i-1}; a_i \right)} - f_i \left( x_i, \bfo_{i-1}, a_i \right) + \left\langle f_i \left( x_i, \bfo_{i-1}, \cdot \right), \pi_{i}^* \left( x_i, \bfo_{i-1}; \cdot \right) \right\rangle_{\lambda_{\bbA}} \right\}.
\]
For each $i \in [n]$, let $\nu_i (\cdot) \in \Delta \left( \bbX \times \bbO^{i-1} \times \bbA \right)$ denote the joint distribution of $\left( X_i, \bfO_{i-1}, A_i \right)$ induced by the adaptive data collection scheme described in Section \ref{sec:problem_formulation}. Then, we arrive at the following result whose proof is deferred to Appendix \ref{subsec:proof_prop:mean_variance_perturbed_ipw}:

\begin{prop}
\label{prop:mean_variance_perturbed_ipw}
For any given collection $f := \left( f_i \in L^2 \left( \nu_i \right) : i \in [n] \right)$ of auxiliary deterministic functions, it holds that $\bbE_{\calI^*} \left[ \hat{\tau}_{n}^{f} \left( \bfO_n \right) \right] = \tau \left( \calI^* \right)$. Furthermore, if
\begin{equation}
    \label{eqn:prop:mean_variance_perturbed_ipw_v1}
    \begin{split}
        \left\langle f_i \left( x, \bfo_{i-1}, \cdot \right), \pi_{i}^* \left( x, \bfo_{i-1}; \cdot \right) \right\rangle_{\lambda_{\bbA}} = 0,\ \forall \left( x, \bfo_{i-1} \right) \in \bbX \times \bbO^{i-1}
    \end{split}
\end{equation}
for each $i \in [n]$, then
\begin{align}
    \label{eqn:prop:mean_variance_perturbed_ipw_v2}
    & n \cdot \textnormal{\textsf{Var}}_{\calI^*} \left[ \hat{\tau}_{n}^{f} \left( \bfO_n \right) \right] 
    = \textnormal{\textsf{Var}}_{X \sim \Xi^*} \left[ \left\langle g ( X, \cdot ), \mu^* ( X, \cdot ) \right\rangle_{\lambda_{\bbA}} \right] + \left\| \sigma \right\|_{(n)}^2 \\
    &+ \frac{1}{n} \sum_{i=1}^{n} \bbE_{\calI^*} \left[ \left\{ \frac{g \left( X_i, A_i \right) \mu^* \left( X_i, A_i \right)}{\pi_{i}^* \left( X_i, \bfO_{i-1}; A_i \right)} - \left\langle g \left( X_i, \cdot \right), \mu^* \left( X_i, \cdot \right) \right\rangle_{\lambda_{\bbA}} - f_i \left( X_i, \bfO_{i-1}, A_i \right) \right\}^2 \right]. \nonumber
\end{align}
\end{prop}

\indent From the decomposition \eqref{eqn:prop:mean_variance_perturbed_ipw_v2} of the variance of the perturbed \textsf{IPW} estimate $\hat{\tau}_{n}^{f} \left( \bfO_n \right)$, one can observe that the only term which depends on the collection of auxiliary functions $f$ is the third term. 
More importantly, the third term is equal to zero if and only if
\begin{equation}
    \label{eqn:optimal_auxiliary_functions}
    \begin{split}
        f_{i} \left( x, \bfo_{i-1}, a \right) = f_{i}^* \left( x, \bfo_{i-1}, a \right) := \frac{g \left( x, a \right) \mu^* \left( x, a \right)}{\pi_{i}^* \left( x, \bfo_{i-1}; a \right)} - \left\langle g (x, \cdot), \mu^* (x, \cdot) \right\rangle_{\lambda_{\bbA}}.
    \end{split}
\end{equation}
The collection of functions $f^* := \left( f_{i}^* \in L^2 \left( \nu_i \right) : i \in [n] \right)$ minimizing the third term in the right-hand side of the equation \eqref{eqn:prop:mean_variance_perturbed_ipw_v2} yields the \emph{oracle estimator} $\hat{\tau}_{n}^{f^*} \left( \cdot \right): \bbO^n \to \bbR$ 
\begin{equation}
    \label{eqn:oracle_IPW_estimator}
    \begin{split}
        \hat{\tau}_{n}^{f^*} \left( \bfO_n \right) = \frac{1}{n} \sum_{i=1}^{n} \left\{ \frac{g \left( X_i, A_i \right) \left\{ Y_i - \mu^* \left( X_i, A_i \right) \right\}}{\pi_{i}^* \left( X_i, \bfO_{i-1}; A_i \right)} + \left\langle g \left( X_i, \cdot \right), \mu^* \left( X_i, \cdot \right) \right\rangle_{\lambda_{\bbA}} \right\},
    \end{split}
\end{equation}
whose variance is given by
\begin{equation}
    \label{eqn:optimal_variance}
    \begin{split}
        n \cdot \textsf{Var}_{\calI^*} \left[ \hat{\tau}_{n}^{f^*} \left( \bfO_n \right) \right] = v_{*}^2 := \textsf{Var}_{X \sim \Xi^*} \left[ \left\langle g ( X, \cdot ), \mu^* ( X, \cdot ) \right\rangle_{\lambda_{\bbA}} \right] + \left\| \sigma \right\|_{(n)}^2.
    \end{split}
\end{equation}

\subsection{The class of augmented \textsf{IPW} estimators}
\label{subsubsec:aipw_estimator}

Since the treatment effect $\mu^*$ is not revealed to the statistician in \eqref{eqn:oracle_IPW_estimator}, it is impossible to exactly compute the oracle estimate $\hat{\tau}_{n}^{f^*} \left( \bfO_n \right)$ using only the observational dataset $\bfO_n$. Therefore, one natural remedy would be the following two-stage procedure, which is referred to as the \emph{augmented inverse propensity weighting} (\textsf{AIPW}) estimator or the \emph{doubly-robust} (\textsf{DR}) estimator \cite{dudik2011doubly, robins1994estimation, van2008construction, hadad2021confidence, zhan2021off, howard2021time}: (\romannumeral 1) we first compute a sequence of estimates $\left\{ \hat{\mu}_{i} \left( \bfO_{i-1} \right) \in \left( \bbX \times \bbA \to \bbR \right) : i \in [n] \right\}$ of the treatment effect; and then (\romannumeral 2) we plug-in these estimates to the equation \eqref{eqn:oracle_IPW_estimator} to construct an approximation to the ideal estimate $\hat{\tau}_{n}^{f^*} \left( \bfO_n \right)$. We summarize this two-stage procedure in Algorithm \ref{alg:aipw_estimator}.
\medskip

\begin{algorithm}[t]
\caption{Meta-algorithm: augmented inverse propensity weighting (\textsf{AIPW}) estimator.}
\label{alg:aipw_estimator}
\begin{algorithmic}[1]
    \Require{the dataset $\calD = \left\{ \left( X_i, A_i, Y_i \right) \in \bbO: i \in [n] \right\}$ and an evaluation function $g: \bbX \times \bbA \to \bbR$.}
    \State For each step $i \in [n]$, we compute an estimate $\hat{\mu}_{i} \left( \bfO_{i-1} \right) \in \left( \bbX \times \bbA \to \bbR \right)$ of the treatment effect based on the sample trajectory $\bfO_{i-1}$ up to the $(i-1)$-th step. \textcolor{blue}{\texttt{// Implement Algorithm \ref{alg:online_regression_v1} as a subroutine;}}
    \State Consider the \textsf{AIPW} estimator (a.k.a., the \emph{doubly-robust} (\textsf{DR}) estimator) $\hat{\tau}_{n}^{\textsf{AIPW}} \left( \cdot \right): \bbO^n \to \bbR$:
    \begin{equation}
        \label{eqn:alg:aipw_estimator_v1}
        \begin{split}
            \hat{\tau}_{n}^{\textsf{AIPW}} \left( \bfo_n \right) := \ & \frac{1}{n} \sum_{i=1}^{n} \hat{\Gamma}_{i} \left( \bfo_{i} \right),
        \end{split}
    \end{equation}
    where the objects being averaged are the \textsf{AIPW} scores $\hat{\Gamma}_{i} (\cdot): \bbO^{i} \to \bbR$ is defined by
    \begin{equation}
        \label{eqn:alg:aipw_estimator_v2}
        \begin{split}
            \hat{\Gamma}_{i} \left( \bfo_{i} \right) := \frac{g \left( x_i, a_i \right)}{\pi_{i}^* \left( x_i, \bfo_{i-1}; a_i \right)} \left\{ y_i - \hat{\mu}_{i} \left( \bfo_{i-1} \right) \left( x_i, a_i \right) \right\} + \left\langle g \left( x_i, \cdot \right), \hat{\mu}_{i} \left( \bfo_{i-1} \right) \left( x_i, \cdot \right) \right\rangle_{\lambda_{\bbA}}.
        \end{split}
    \end{equation}
    \State \Return the \textsf{AIPW} estimate $\hat{\tau}_{n}^{\textsf{AIPW}} \left( \bfO_n \right)$.
\end{algorithmic}
\end{algorithm}

\indent We pause here to compare the setting of our problem and algorithms with the most relevant work \cite{mou2022off}. We focus on off-policy estimation with adaptively collected data, which is technically more challenging compared to i.i.d.~data considered in \cite{mou2022off}. In the case of i.i.d.~data collection scheme, \cite{mou2022off} proposed a natural approach to construct a class of two-stage estimators as follows: (a) compute an estimate $\hat{\mu}$ of the treatment effect $\mu^*$ utilizing part of the dataset; and (b) substitute this estimate in the equation \eqref{eqn:oracle_IPW_estimator} of the oracle estimator. Note that the authors use the \emph{cross-fitting approach} \cite{chernozhukov2017double, chernozhukov2018double}, which allows to make full use of data to maintain efficiency and statistical power of machine learning algorithms for estimation of nuisance parameters while reducing over-fitting bias. However, the cross-fitting strategy heavily relies on the i.i.d.~nature of the data collection mechanism and therefore one cannot use it in the setting with adaptively collected data. Instead, one constructs an estimate $\hat{\mu}_i$ of the treatment effect $\mu^*$ based on the sample trajectory $\bfO_{i-1}$ at each stage, and then substitute these estimates into the equation \eqref{eqn:oracle_IPW_estimator}. This is one of main contributions to address the adaptive nature of our data generating mechanism. We will make use of the framework of online learning to construct a sequence of estimates for the treatment effect $\mu^*$.

\subsection{Theoretical guarantees of Algorithm \ref{alg:aipw_estimator}}
\label{subsec:theoretical_guarantees_aipw_estimator}

Throughout this section, we provide statistical guarantees for the class of \textsf{AIPW} estimators for dealing with estimation of the off-policy value \eqref{eqn:off_policy_value_functional}. The main result of this section can be summarized as the following non-asymptotic upper bound on the mean-squared error (\textsf{MSE}) of Algorithm \ref{alg:aipw_estimator}:

\begin{thm} [Non-asymptotic upper bound on the \textsf{MSE} of the \textsf{AIPW} estimator]
\label{thm:mse_aipw}
Given any sequence of estimates $\left\{ \hat{\mu}_{i} \left( \bfO_{i-1} \right) \in \left( \bbX \times \bbA \to \bbR \right): i \in [n] \right\}$ for the treatment effect $\mu^*$, the \textnormal{\textsf{AIPW}} estimator \eqref{eqn:alg:aipw_estimator_v1} has the \textnormal{\textsf{MSE}} bounded above by
\begin{equation}
    \label{eqn:thm:mse_aipw_v1}
    \begin{split}
        \bbE_{\calI^*} \left[ \left\{ \hat{\tau}_{n}^{\textnormal{\textsf{AIPW}}} \left( \bfO_n \right) - \tau \left( \calI^* \right) \right\}^2 \right] \leq \frac{1}{n} \left\{ v_{*}^2 + \frac{1}{n} \sum_{i=1}^{n} \bbE_{\calI^*} \left[ \frac{g^2 \left( X_i, A_i \right) \left\{ \hat{\mu}_{i} \left( \bfO_{i-1} \right) \left( X_i, A_i \right) - \mu^* \left( X_i, A_i \right) \right\}^2}{\left( \pi_{i}^* \right)^2 \left( X_i, \bfO_{i-1}; A_i \right)} \right] \right\}.
    \end{split}
\end{equation}
\end{thm}

\indent We notice that the non-asymptotic upper bound \eqref{eqn:thm:mse_aipw_v1} on the \textsf{MSE} for the class of \textsf{AIPW} estimators \eqref{eqn:alg:aipw_estimator_v1} consists of two terms, both of which have natural interpretations. Here, the first term $v_{*}^2$ corresponds to the optimal variance \eqref{eqn:optimal_variance} achievable by the oracle estimator, and the second term
\begin{equation}
    \label{eqn:treatment_effect_estimation_error}
    \begin{split}
        \frac{1}{n} \sum_{i=1}^{n} \bbE_{\calI^*} \left[ \frac{g^2 \left( X_i, A_i \right) \left\{ \hat{\mu}_{i} \left( \bfO_{i-1} \right) \left( X_i, A_i \right) - \mu^* \left( X_i, A_i \right) \right\}^2}{\left( \pi_{i}^* \right)^2 \left( X_i, \bfO_{i-1}; A_i \right)} \right]
    \end{split}
\end{equation}
measures the average estimation error of the estimates $\left\{ \hat{\mu}_{i} \left( \bfO_{i-1} \right) \in \left( \bbX \times \bbA \to \bbR \right): i \in [n] \right\}$ of $\mu^*$. Of primary interest to us is a subsequent upper bounding argument based on the upper bound \eqref{eqn:thm:mse_aipw_v1} on the \textsf{MSE} in the finite-sample regime: in particular, to minimize the right-hand side of the bound \eqref{eqn:thm:mse_aipw_v1}, one needs to choose a sequence of estimates $\left\{ \hat{\mu}_{i} \left( \bfO_{i-1} \right) \in \left( \bbX \times \bbA \to \bbR \right): i \in [n] \right\}$ which minimizes the second term \eqref{eqn:treatment_effect_estimation_error}.

\subsection{Reduction to online non-parametric regression}
\label{subsec:reduction_online_np_regression}

We now focus on constructing a sequence of estimates $\left\{ \hat{\mu}_{i} \left( \bfO_{i-1} \right) \in \left( \bbX \times \bbA \to \bbR \right): i \in [n] \right\}$ for the treatment effect, and upper bounding the estimation error \eqref{eqn:treatment_effect_estimation_error} in the \textsf{MSE} bound \eqref{eqn:thm:mse_aipw_v1}. To this end, we borrow ideas from the literature of online non-parametric regression \cite{rakhlin2014online}.
\medskip

\indent To begin with, we consider a turn-based game with $n$ rounds between the learner and the environment: see Algorithm \ref{alg:online_regression_v1} for the details.  Then, one can readily observe for any $\mu (\cdot, \cdot) : \bbX \times \bbA \to \bbR$, we have 
\begin{equation}
    \label{eqn:property_online_regression_v1_loss}
    \begin{split}
        \bbE_{\calI^*} \left[ \left. l_i (\mu) \right| \left( \calH_{i-1}, X_i, A_i \right) \right] = \frac{g^2 \left( X_i, A_i \right)}{\left( \pi_{i}^* \right)^2 \left( X_i, \bfO_{i-1}; A_i \right)} \left[ \sigma^2 \left( X_i, A_i \right) + \left\{ \mu \left( X_i, A_i \right) - \mu^* \left( X_i, A_i \right) \right\}^2 \right].
    \end{split}
\end{equation}
In the current turn-based game framework, our natural goal is to minimize the learner's static regret against the \emph{best fixed action in hindsight} belonging to a pre-specified function class $\calF \subseteq \left( \bbX \times \bbA \to \bbR \right)$:
\begin{equation}
    \label{eqn:defi_regret_v1}
    \begin{split}
        \textsf{Regret} \left( n, \calF; \calA \right) := \sum_{i=1}^{n} l_i \left\{ \hat{\mu}_i \left( \bfO_{i-1} \right) \right\} - \inf \left\{ \sum_{i=1}^{n} l_i (\mu): \mu \in \calF \right\},
    \end{split}
\end{equation}
where $\calA$ denotes the learner's online non-parametric regression algorithm that returns a sequence of estimates $\left\{ \hat{\mu}_i \left( \bfO_{i-1} \right): i \in [n]\right\}$ for the treatment effect $\mu^*$. Then, one can prove the following oracle inequality that demystifies a relationship between the estimation problem of the off-policy value and the online non-parametric regression protocol. See Appendix \ref{subsec:proof_thm:online_regression_oracle_ineq_v1} for the proof.

\begin{algorithm}[t]
\caption{Online non-parametric regression protocol for estimation of the treatment effect.}
\label{alg:online_regression_v1}
\begin{algorithmic}[1]
    \Require{the number of rounds $n \in \bbN$.}
    \For{$i = 1, 2, \cdots, n$,}
        \State The learner selects a point $\hat{\mu}_{i} \left( \bfO_{i-1} \right) \in \left( \bbX \times \bbA \to \bbR \right)$ based on the sample trajectory $\bfO_{i-1}$;
        \State The environment then picks a loss function $l_i (\cdot): \left( \bbX \times \bbA \to \bbR \right) \to \bbR$ defined as
        \begin{equation}
            \label{eqn:online_regression_v1_loss}
            l_i (\mu) := \frac{g^2 \left( X_i, A_i \right)}{\left( \pi_{i}^* \right)^2 \left( X_i, \bfO_{i-1}; A_i \right)} \left\{ Y_i - \mu \left( X_i, A_i \right) \right\}^2,\ \forall \mu (\cdot, \cdot) \in \left( \bbX \times \bbA \to \bbR \right).
        \end{equation}
    \EndFor
    \State \Return the sequence of estimates $\left\{ \hat{\mu}_i \left( \bfO_{i-1} \right)\in \left( \bbX \times \bbA \to \bbR \right): i \in [n] \right\}$ of the treatment effect.
\end{algorithmic}
\end{algorithm}

\begin{thm} [Oracle inequality for the class of \textsf{AIPW} estimators]
\label{thm:online_regression_oracle_ineq_v1}
The \textnormal{\textsf{AIPW}} estimator \eqref{eqn:alg:aipw_estimator_v1} utilizing the sequence of estimates $\left\{ \hat{\mu}_i \left( \bfO_{i-1} \right) \in \left( \bbX \times \bbA \to \bbR \right): i \in [n] \right\}$ for the treatment effect $\mu^*$ produced by the online non-parametric regression algorithm $\calA$ enjoys the following upper bound on the \textnormal{\textsf{MSE}}:
\begin{equation}
    \label{eqn:thm:online_regression_oracle_ineq_v1_v1}
    \begin{split}
        \bbE_{\calI^*} \left[ \left\{ \hat{\tau}_{n}^{\textnormal{\textsf{AIPW}}} \left( \bfO_n \right) - \tau \left( \calI^* \right) \right\}^2 \right] \leq \frac{1}{n} \left( v_{*}^2 + \frac{1}{n} \bbE_{\calI^*} \left[ \textnormal{\textsf{Regret}} \left( n, \calF; \calA \right) \right] + \inf \left\{ \left\| \mu - \mu^* \right\|_{(n)}^2 : \mu \in \calF \right\} \right).
    \end{split}
\end{equation}
\end{thm}

\indent A few remarks are in order. Apart from the optimal variance $v_{*}^2$, the right-hand side of the bound \eqref{eqn:thm:online_regression_oracle_ineq_v1_v1} contains two additional terms: (\romannumeral 1) the expected regret relative to the number of rounds $n$, where the expected value is taken over $\bfO_n \sim \bbP_{\calI^*}^n (\cdot)$; and (\romannumeral 2) the approximation error under the $\left\| \cdot \right\|_{(n)}$-norm. For any fixed function class $\calF \subseteq \left( \bbX \times \bbA \to \bbR \right)$, if we consider the large sample size regime, that is, the sample size $n$ is sufficiently large, then one can see that the asymptotic variance of the \textsf{AIPW} estimator \eqref{eqn:alg:aipw_estimator_v1} is asymptotically the same as $v_{*}^2 + \inf \left\{ \left\| \mu - \mu^* \right\|_{(n)}^2 : \mu \in \calF \right\}$, when the online non-parametric regression algorithm $\calA$ exhibits a \emph{no-regret learning dynamics}, i.e., $\bbE_{\calI^*} \left[ \textsf{Regret} \left( n, \calF; \calA \right) \right] = o (n)$ as $n \to \infty$. Consequently, the \textsf{AIPW} estimator \eqref{eqn:alg:aipw_estimator_v1} might suffer from an efficiency loss which depends on how well the (unknown) treatment effect $\mu^*$ can be approximated by a member of the function class $\calF \subseteq \left( \bbX \times \bbA \to \bbR \right)$ under the $\left\| \cdot \right\|_{(n)}$-norm. Therefore, any contribution to the \textsf{MSE} bound of the \textsf{AIPW} estimator \eqref{eqn:alg:aipw_estimator_v1} \emph{in addition to} the efficient variance $v_{*}^2$ primarily relies on the approximation error associated with approximating the treatment effect $\mu^*$ utilizing a provided function class $\calF$.

\subsection{Consequences for particular outcome models}
\label{subsec:consequences_particular_outcome_models}

The main goal of this section is to illustrate the consequences of our general theory developed in Section \ref{sec:AIPW_non_asymptotic_guarantees} so far for several concrete classes of outcome models. Throughout this section, we consider the case for which $\bbY = \left[ -L, L \right]$ for some constant $L \in \left( 0, +\infty \right)$, and impose the following condition:

\begin{assumption} [Strict overlap condition]
\label{assumption:strict_overlap_condition}
\normalfont{
The likelihood ratios are uniformly bounded by a universal constant $B \in \left( 0, +\infty \right)$, i.e., for every $i \in [n]$,
\begin{equation}
    \label{eqn:strict_overlap_condition}
    \begin{split}
        \left| \frac{g \left( X_i, A_i \right)}{\pi_{i}^* \left( X_i, \bfO_{i-1}; A_i \right)} \right| \leq B \quad \bbP_{\calI^*}^n\textnormal{-almost surely.}
    \end{split}
\end{equation}
}
\end{assumption}

We note that Assumption \ref{assumption:strict_overlap_condition} is often referred to as the \emph{strict overlap condition} in the literature of causal inference \cite{hirano2003efficient, khan2010irregular, yang2018asymptotic, lei2021distribution, d2021overlap}. At this point, we emphasize that Assumption \ref{assumption:strict_overlap_condition} is necessary to produce main consequences of the oracle inequality for the class of \textsf{AIPW} estimators (Theorem \ref{thm:online_regression_oracle_ineq_v1}) that we discuss in the ensuing subsections: Theorems \ref{thm:regret_bound_ogd_v1}, \ref{thm:regret_bound_ogd_v2}, and the arguments throughout Appendix \ref{subsec:general_function_approximation}.

\subsubsection{Tabular case of the outcome model}
\label{subsubsec:tabular_outcome_model}

We embark on our discussion about consequences of our theory established in Sections \ref{subsec:theoretical_guarantees_aipw_estimator} and \ref{subsec:reduction_online_np_regression} for one of the simplest case of the outcome model satisfying the following assumption.

\begin{assumption} [Tabular setting of the outcome model]
\label{assumption:finite_state_action_space}
\normalfont{
The state-action space $\bbX \times \bbA$ is a finite set. 
}
\end{assumption}

\indent If we compute the gradient of the loss function \eqref{eqn:online_regression_v1_loss}, we have
\begin{equation}
    \label{eqn:gradient_online_regression_v1_loss}
    \begin{split}
        \nabla l_i (\mu) = \frac{2 g^2 \left( X_i, A_i \right)}{\left( \pi^* \right)^2 \left( X_i, \bfO_{i-1}; A_i \right)} \left\{ \mu \left( X_i, A_i \right) - Y_i \right\} \delta_{\left( X_i, A_i \right)},\ \forall \mu \in \bbR^{\bbX \times \bbA},
    \end{split}
\end{equation}
where $\delta_{\left( X_i, A_i \right)} \in \bbR^{\bbX \times \bbA}$ denotes the point-mass vector at the $i$-th state-action pair in the sample trajectory, i.e., $\delta_{\left( X_i, A_i \right)} (x, a) := 1$ if $\left( x, a \right) = \left( X_i, A_i \right)$; $\delta_{\left( X_i, A_i \right)} (x, a) := 0$ otherwise.

\begin{algorithm}[h!]
\caption{Online gradient descent (\textsf{OGD}) algorithm for the finite state-action space.}
\label{alg:ogd_v1}
\begin{algorithmic}[1]
    \Require{the function class $\calF \subseteq \left[ -L, L \right]^{\bbX \times \bbA}$, the total number of rounds $n \in \bbN$, and a sequence of learning rates $\left\{ \eta_i \in \left( 0, +\infty \right): i \in [n-1] \right\}$.}
    \State{We first choose an initial point $\hat{\mu}_1 (\varnothing) \in \calF$ arbitrarily;}
    \For{$i = 1, 2, \cdots, n-1$,}
        \State Observe a triple $\left( X_i, A_i, Y_i \right) \in \bbO$; 
        \State Update $\hat{\mu}_{i+1} \left( \bfO_i \right) \in \calF$ according to the following \textsf{OGD} update rule:
        \begin{equation}
            \label{eqn:ogd_v1_update_rule}
            \begin{split}
                \hat{\mu}_{i+1} \left( \bfO_i \right) = \ & \Pi_{\calF} \left[ \hat{\mu}_i \left( \bfO_{i-1} \right) - \eta_i \nabla l_i \left\{ \hat{\mu}_i \left( \bfO_{i-1} \right) \right\} \right] \\
                = \ & \Pi_{\calF} \left[ \hat{\mu}_i \left( \bfO_{i-1} \right) - \frac{2 \eta_i \cdot g^2 \left( X_i, A_i \right)}{\left( \pi_{i}^* \right)^2 \left( X_i, \bfO_{i-1}; A_i \right)} \left\{ \hat{\mu}_i \left( \bfO_{i-1} \right) - Y_i \right\} \delta_{\left( X_i, A_i \right)} \right],
            \end{split}
        \end{equation}
        where $\Pi_{\calF} [ \cdot ]: \bbR^{\bbX \times \bbA} \to \calF$ denotes the projection map of $\bbR^{\bbX \times \bbA}$ onto the function space $\calF$.
    \EndFor
    \State \Return the sequence of estimates $\left\{ \hat{\mu}_i \left( \bfO_{i-1} \right) \in \calF : i \in [n] \right\}$ of the treatment effect $\mu^*$.
\end{algorithmic}
\end{algorithm}

Now, it's time to put forward an online contextual learning algorithm aimed at producing a sequence of estimates for the treatment effect with a no-regret learning guarantee. The online non-parametric regression problem can be resolved through standard online convex optimization (\textsf{OCO}) algorithms for the tabular case. In particular, we make use of the online gradient descent (\textsf{OGD}) algorithm (see Algorithm \ref{alg:ogd_v1}) as a sub-routine of Algorithm \ref{alg:aipw_estimator}. By leveraging standard results on regret analysis of \textsf{OCO} algorithms, one can establish the following regret bound, which guarantees a no-regret learning dynamics of Algorithm \ref{alg:ogd_v1}.

\begin{thm} [Regret guarantee of Algorithm \ref{alg:ogd_v1}]
\label{thm:regret_bound_ogd_v1}
With Assumptions \ref{assumption:strict_overlap_condition} and \ref{assumption:finite_state_action_space} in hand, the \textnormal{\textsf{OGD}} algorithm (Algorithm \ref{alg:ogd_v1}) with learning rates $\left\{ \eta_i := \frac{\textnormal{diam} (\calF)}{4 LB^2 \sqrt{i}}: i \in [n] \right\}$ guarantees
\begin{equation}
    \label{eqn:thm:regret_bound_ogd_v1_v1}
    \begin{split}
        \textnormal{\textsf{Regret}} \left( n, \calF; \textnormal{\textsf{OGD}} \right) \leq 6 L B^2 \textnormal{\textsf{diam}} (\calF) \cdot \sqrt{n} \quad \bbP_{\calI^*}^n\textnormal{-almost surely,}
    \end{split}
\end{equation}
where $\textnormal{\textsf{diam}} (\calF) := \sup \left\{ \left\| \mu \right\|_2 : \mu \in \calF \right\}$ denotes the diameter of $\calF \subseteq \left[ -L, L \right]^{\bbX \times \bbA}$.
\end{thm}

See Appendix \ref{subsec:proof_thm:regret_bound_ogd_v1} for the proof of Theorem \ref{thm:regret_bound_ogd_v1}. By combining the regret guarantee \eqref{eqn:thm:regret_bound_ogd_v1_v1} of Algorithm \ref{alg:ogd_v1} together with the \textsf{MSE} upper bound \eqref{eqn:thm:online_regression_oracle_ineq_v1_v1} in Theorem \ref{thm:online_regression_oracle_ineq_v1}, one can establish a concrete upper bound on the \textsf{MSE} of the \textsf{AIPW} estimator \eqref{eqn:alg:aipw_estimator_v1} by using Algorithm \ref{alg:ogd_v1} to produce a sequence of estimates for the treatment effect $\mu^*$. 

\subsubsection{Linear function approximation}
\label{subsubsec:linear_function_approximation}

We move on to outcome models where the state-action space $\bbX \times \bbA$ can be infinite. We begin with the simplest case: the class of linear outcome functions. We let $\phi (\cdot, \cdot): \bbX \times \bbA \to \bbR^d$ be a \emph{known feature map} such that $\sup \left\{ \left\| \phi (x, a) \right\|_{2}: (x, a) \in \bbX \times \bbA \right\} \leq 1$, and then consider the functions that are linear in this representation: $f_{\bftheta} (\cdot, \cdot): \bbX \times \bbA \to \bbR$, where $f_{\bftheta} (x, a) := \bftheta^{\top} \phi (x, a)$ for some parameter vector $\bftheta \in \bbR^d$. Given a radius $R > 0$, we define the function class
\begin{equation}
    \label{eqn:linear_function_class}
    \begin{split}
        \calF_{\textsf{lin}} := \left\{ f_{\bftheta} (\cdot, \cdot) \in \left( \bbX \times \bbA \to \bbR \right): \bftheta \in \Theta := \overline{\bbB \left( \bfzero_d; R \right)} \right\},
    \end{split}
\end{equation}
where $\overline{\bbB \left( \bfzero_d; R \right)} := \left\{ \bfu \in \bbR^d: \left\| \bfu \right\|_2 \leq R \right\}$. With this linear function approximation framework, let us consider the following \textsf{OCO} model: at the $i$-th stage,
\begin{enumerate} [label = (\roman*)]
    \item the learner first chooses a point $\hat{\theta}_i \left( \bfO_{i-1} \right) \in \Theta$;
    \item the environment then picks a loss function $\calL_i (\cdot): \Theta \to \bbR$ defined as
    \begin{equation}
        \label{eqn:linear_approx_online_learning_loss}
        \begin{split}
            \calL_i (\bftheta) := \frac{g^2 \left( X_i, A_i \right)}{\left( \pi_{i}^* \right)^2 \left( X_i, \bfO_{i-1}; A_i \right)} \left\{ Y_i - \bftheta^{\top} \phi \left( X_i, A_i \right) \right\}^2, \ \forall \bftheta \in \Theta.
        \end{split}
    \end{equation}
\end{enumerate}
Here, our goal is to produce a sequence of estimates $\left\{ \hat{\mu}_i \left( \bfO_{i-1} \right) := \left\{ \hat{\bftheta}_i \left( \bfO_{i-1} \right) \right\}^{\top} \phi \in \calF_{\textnormal{lin}}: i \in [n] \right\}$ for the treatment effect $\mu^*$ after $n$ rounds of the above-mentioned \textsf{OCO} protocol which minimizes the learner's regret against the \emph{best fixed action in hindsight}:
\[
    \begin{split}
        \textsf{Regret} \left( n, \calF_{\textsf{lin}}; \calA \right) = \ & \sum_{i=1}^{n} l_i \left\{ \hat{\mu}_i \left( \bfO_{i-1} \right) \right\} - \inf \left\{ \sum_{i=1}^{n} l_i (\mu) : \mu \in \calF \right\} \\
        = \ & \sum_{i=1}^{n} \calL_i \left\{ \hat{\bftheta}_i \left( \bfO_{i-1} \right) \right\} - \inf \left\{ \sum_{i=1}^{n} \calL_i (\bftheta) : \bftheta \in \Theta \right\},
    \end{split}
\]
where $\calA$ is the learner's \textsf{OCO} algorithm whose output is a sequence $\left\{ \hat{\bftheta}_i \left( \bfO_{i-1} \right) \in \Theta: i \in [n]\right\}$ of parameters. If we compute the gradient of the loss function \eqref{eqn:linear_approx_online_learning_loss}, one has
\begin{equation}
    \label{eqn:gradient_linear_approx_online_learning_loss}
    \begin{split}
        \nabla_{\bftheta} \calL_i (\bftheta) = \frac{2 g^2 \left( X_i, A_i \right)}{\left( \pi_{i}^* \right)^2 \left( X_i, \bfO_{i-1}; A_i \right)} \left\{ \bftheta^{\top} \phi \left( X_i, A_i \right) - Y_i \right\} \phi \left( X_i, A_i \right).
    \end{split}
\end{equation}
For the current linear function approximation setting, we implement the \textsf{OGD} algorithm (Algorithm \ref{alg:ogd_v2}) as a sub-routine of Algorithm \ref{alg:aipw_estimator}. Using the same arguments as in Section \ref{subsubsec:tabular_outcome_model}, one can reproduce the following regret guarantee of Algorithm \ref{alg:ogd_v2} whose proof is available at Appendix \ref{subsec:proof_thm:regret_bound_ogd_v2}.

\begin{algorithm}[t]
\caption{Online gradient descent (\textsf{OGD}) algorithm for linear function approximation.}
\label{alg:ogd_v2}
\begin{algorithmic}[1]
    \Require{the radius $R \in \left( 0, +\infty \right)$ of the parameter space $\Theta$, the number of rounds $n \in \bbN$, and a sequence of learning rates $\left\{ \eta_i \in \left( 0, +\infty \right): i \in [n-1] \right\}$.}
    \State{We first choose an arbitrary initial point $\hat{\bftheta}_1 (\varnothing) \in \Theta$, where $\Theta := \overline{\bbB \left( \bfzero_d; R \right)}$;}
    \For{$i = 1, 2, \cdots, n-1$,}
        \State Observe a triple $\left( X_i, A_i, Y_i \right) \in \bbO$; 
        \State Update $\hat{\bftheta}_{i+1} \left( \bfO_i \right) \in \Theta$ according to the following \textsf{OGD} update rule:
        \begin{equation}
            \label{eqn:ogd_v2_update_rule}
            \begin{split}
                \hat{\bftheta}_{i+1} \left( \bfO_i \right) = \Pi_{\Theta} \left[ \hat{\bftheta}_i \left( \bfO_{i-1} \right) - \eta_i \nabla_{\bftheta} \calL_i \left\{ \hat{\bftheta}_i \left( \bfO_{i-1} \right) \right\} \right],
            \end{split}
        \end{equation}
        where $\Pi_{\Theta} [ \cdot ]: \bbR^d \to \Theta$ denotes the projection map of $\bbR^d$ onto the parameter space $\Theta$.
    \EndFor
    \State \Return the estimates $\left\{ \hat{\mu}_i \left( \bfO_{i-1} \right) := \left\{ \hat{\bftheta}_i \left( \bfO_{i-1} \right) \right\}^{\top} \phi \in \calF_{\textnormal{lin}}: i \in [n] \right\}$ of the treatment effect.
\end{algorithmic}
\end{algorithm}

\begin{thm} [Regret guarantee of Algorithm \ref{alg:ogd_v2}]
\label{thm:regret_bound_ogd_v2}
With Assumption \ref{assumption:strict_overlap_condition}, the \textnormal{\textsf{OGD}} algorithm (Algorithm \ref{alg:ogd_v2}) with learning rates $\left\{ \eta_i := \frac{R}{B^2 (L+R) \sqrt{i}}: i \in [n] \right\}$ guarantees
\begin{equation}
    \label{eqn:thm:regret_bound_ogd_v2_v1}
    \begin{split}
        \textnormal{\textsf{Regret}} \left( n, \calF_{\textnormal{\textsf{lin}}}; \textnormal{\textsf{OGD}} \right) \leq 6 B^2 R (L+R) \sqrt{n} \quad \bbP_{\calI^*}^n\textnormal{-almost surely.}
    \end{split}
\end{equation}
\end{thm}

\paragraph{General function approximation}
As a final step, it's time to demonstrate consequences of our general theory established in Sections \ref{subsec:theoretical_guarantees_aipw_estimator} and \ref{subsec:reduction_online_np_regression} for the case of general function approximation: the function class $\calF \subseteq \left( \bbX \times \bbA \to \left[ -L, L \right] \right)$ can be arbitrarily chosen. Our further discussion this case heavily relies on the basic theory of online non-parametric regression from \cite{rakhlin2014online} whose technical details are rather long and complicated. So, we defer our inspection on the case of general function approximation to Appendix \ref{subsec:general_function_approximation}.

\section{Lower bounds: local minimax risk}
\label{sec:lower_bounds}

We turn our attention to a local minimax lower bound for estimating the off-policy value $\tau^* = \tau \left( \calI^* \right)$. Here, we aim at establishing lower bounds that hold uniformly over all estimators that are permitted to know both the propensity scores $\left\{ \pi_{i}^* \left( X_i, \bfO_{i-1}; A_i \right) : i \in [n] \right\}$ and the evaluation function $g$. We assume the existence of a constant $K \in \left[ 1, +\infty \right)$ and \emph{reference Markov policies} $\left\{ \overline{\Pi}_i : \bbX \to \Delta (\bbA): i \in [n] \right\}$ such that $\overline{\Pi}_i \left( \left. \cdot \right| x \right) \ll \lambda_{\bbA} (\cdot)$ for $(x, i) \in \bbX \times [n]$, and
\begin{equation}
    \label{eqn:assumption_reference_Markov_policies}
    \frac{1}{K} \leq \frac{\overline{\pi}_i \left( x, a \right)}{\pi_{i}^* \left( x, \bfo_{i-1}; a \right)} \leq K
\end{equation}
for every $\left( x, \bfo_{i-1}, a \right) \in \bbX \times \bbO^{i-1} \times \bbA$, where $\overline{\pi}_i \left( x, \cdot \right) := \frac{\mathrm{d} \overline{\Pi}_i \left( \left. \cdot \right| x \right)}{\mathrm{d} \lambda_{\bbA}}: \bbA \to \bbR_{+}$ for each context $x \in \bbX$. Proximity of behavioral policies to certain Markov policies is often assumed under adaptive data collection procedures. For instance, in \emph{Theorem 1} of \cite{zhan2021off}, the authors assumed that the sequence of behavior policies is \emph{eventually Markov}; see the equation (8) therein. 

\subsection{Instance-dependent local minimax lower bounds}
\label{subsec:instance_dependent_lower_bounds}

For any problem instance $\calI^* = \left( \Xi^*, \Gamma^* \right) \in \bbI$ and an error function $\delta: \bbX \times \bbA \to \bbR_{+}$, we consider the following local neighborhoods:
\[
    \begin{split}
        \calN \left( \Xi^* \right) := \ & \left\{ \Xi \in \Delta (\bbX): \textnormal{KL} \left( \Xi \left\| \Xi^* \right. \right) \leq \frac{1}{n} \right\}; \\
        \calN_{\delta} \left( \Gamma^* \right) := \ & \left\{ \Gamma \in \left( \bbX \times \bbA \to \Delta (\bbY) \right): \left| \mu (\Gamma) (x, a) - \mu \left( \Gamma^* \right) (x, a) \right| \leq \delta (x, a),\ \forall (x, a) \in \bbX \times \bbA \right\},
    \end{split}
\]
where for any given $\Gamma: \bbX \times \bbA \to \Delta (\bbY)$, let $\mu (\Gamma) (x, a) := \int_{\bbY} y \Gamma \left( \left. \mathrm{d} y \right| x, a \right)$ for each $(x, a) \in \bbX \times \bbA$. Our goal is to lower bound the following \emph{local minimax risk}:
\begin{equation}
    \label{eqn:local_minimax_risk}
    \begin{split}
        \calM_n \left( \calC_{\delta} \left( \calI^* \right) \right) := \inf_{\hat{\tau}_n (\cdot): \bbO^n \to \bbR} \left( \sup_{\calI \in \calC_{\delta} \left( \calI^* \right)} \bbE_{\calI} \left[ \left\{ \hat{\tau}_n \left( \bfO_n \right) - \tau \left( \calI \right) \right\}^2 \right] \right),
    \end{split}
\end{equation}
where $\calC_{\delta} \left( \calI^* \right) := \calN \left( \Xi^* \right) \times \calN_{\delta} \left( \Gamma^* \right) \subseteq \bbI$. We now specify some assumptions necessary for lower bounding the local minimax risk \eqref{eqn:local_minimax_risk}. Prior to this, we introduce a new significant notation: given any random variable $Y \in \bbL^4 \left( \Omega, \calF, \bbP \right)$ defined on the underlying probability space $\left( \Omega, \calF, \bbP \right)$, its \emph{$(2, 4)$-moment ratio} is defined as $\left\| Y \right\|_{2 \to 4} := \frac{\sqrt{\bbE \left[ Y^4 \right]}}{\bbE \left[ Y^2 \right]}$.

\begin{assumption}
\label{assumption:mr_v1}
\normalfont{
Let $h (x) := \left\langle g (x, \cdot), \mu^* (x, \cdot) \right\rangle_{\lambda_{\bbA}} - \bbE_{X \sim \Xi^*} \left[ \left\langle g (X, \cdot), \mu^* (X, \cdot) \right\rangle_{\lambda_{\bbA}} \right]$. We assume that 
\[
    H_{2 \to 4} := \left\| h \right\|_{2 \to 4} = \frac{\sqrt{\bbE_{X \sim \Xi^*} \left[ h^4 (X) \right]}}{\bbE_{X \sim \Xi^*} \left[ h^2 (X) \right]} < +\infty.
\]
}
\end{assumption}

\noindent We next make an assumption on a lower bound on the \emph{local neighborhood size}:

\begin{assumption}
\label{assumption:ln}
\normalfont{
The neighborhood function $\delta (\cdot, \cdot): \bbX \times \bbA \to \bbR_{+}$ satisfies the lower bound
\begin{equation}
    \label{eqn:assumption:ln_v1}
    \begin{split}
        \sqrt{n} \cdot \delta (x, a) \geq \frac{\left| g (x, a) \right| \sigma^2 (x, a)}{\overline{\pi}_i (x, a) \left\| \sigma \right\|_{(n)}}
    \end{split}
\end{equation}
for all $(x, a, i) \in \bbX \times \bbA \times [n]$.
}
\end{assumption}

\indent Here, we note that Assumptions \ref{assumption:mr_v1} and \ref{assumption:ln} are analogues of Assumptions (MR) and (LN) considered in \cite{mou2022off}, respectively, for the case of adaptively collected data. Under these assumptions, one can prove the following lower bound on the local minimax risk over the neighborhood of problem instances $\calC_{\delta} \left( \calI^* \right)$:

\begin{thm}
\label{thm:local_minimax_lower_bound}
Under Assumptions \ref{assumption:mr_v1} and \ref{assumption:ln}, the local minimax risk over $\calC_{\delta} \left( \calI^* \right)$ is lower bounded by
\begin{equation}
    \label{eqn:cor:local_minimax_lower_bound_v1_v1}
    \calM_n \left( \calC_{\delta} \left( \calI^* \right) \right) \geq \calC (K) \cdot \frac{v_{*}^2}{n},
\end{equation}
where $\calC (K) > 0$ is a universal constant that only depends on the data coverage constant $K \in \left[ 1, +\infty \right)$ of the reference Markov policies $\left\{ \overline{\Pi}_{i} (\cdot): \bbX \to \Delta (\bbA) : i \in [n] \right\}$ defined in \eqref{eqn:assumption_reference_Markov_policies}.
\end{thm}

\noindent The proof of Theorem \ref{thm:local_minimax_lower_bound} is deferred to Appendix \ref{subsec:proof_thm:local_minimax_lower_bound}. Theorem \ref{thm:local_minimax_lower_bound} delivers the following takeaway message: the term $\frac{v_{*}^2}{n}$ including the sequentially weighted $\ell_2$-norm is indeed the fundamental limit for estimating the linear functional based on adaptively collected data. Our results can be regarded as a generalization of those developed in \cite{mou2022off} for the case of i.i.d.~data. 

\section*{Acknowledgements}

Jeonghwan Lee was partially supported by the Kwanjeong Educational Foundation. Cong Ma was partially supported by the National Science Foundation via grant DMS-2311127.

\newpage

\bibliographystyle{plain}
\bibliography{main.bib}

\newpage

\appendix

\section{Some elementary inequalities and their proofs}
\label{sec:elementary_ineqs}

\indent The following lemma plays a key role in the truncation arguments used in establishing our local minimax lower bounds. In particular, it enables us to make small modifications on a pair of probability dsitributions by conditioning on \emph{good events} of each probability measure, without inducing an irregularly large change in the total variation distance.

\begin{lemma}
\label{lemma:tv_distance_modification}
Let $\left( \mu, \nu \right)$ be any pair of probability measures defined on a common sample space $\left( \Omega, \calF \right)$. Let us consider any two events $A \in \calF$ and $B \in \calF$ satisfying $\min \left\{ \mu (A), \nu (B) \right\} \geq 1 - \epsilon$ for some $\epsilon \in \left[ 0, \frac{1}{4} \right]$. Then, the conditional distributions $\left( \mu | A \right) (\cdot) \in \Delta \left( \Omega, \calF \right)$ and $\left( \nu | B \right) (\cdot) \in \Delta \left( \Omega, \calF \right)$ defined by
\[
    \left( \mu | A \right) (E) := \frac{\mu \left( A \cap E \right)}{\mu (A)} \quad \textnormal{and} \quad \left( \nu | B \right) (E) := \frac{\nu \left( B \cap E \right)}{\nu (B)}
\]
for any event $E \in \calF$, satisfy the bound
\begin{equation}
    \label{eqn:lemma:tv_distance_modification_v1}
    \left| \textnormal{\textsf{TV}} \left( \mu | A, \nu | B \right) - \textnormal{\textsf{TV}} \left( \mu, \nu \right) \right| \leq 2 \epsilon.
\end{equation}
\end{lemma}

\begin{proof} [Proof of Lemma \ref{lemma:tv_distance_modification}]
Due to the triangle inequality for the total variation (\textsf{TV}) distance, it follows that
\begin{equation}
    \label{eqn:proof_lemma:tv_distance_modification_v1_v1}
    \begin{split}
        \textsf{TV} \left( \mu, \nu \right) \leq \textsf{TV} \left( \mu, \mu | A \right) + \textsf{TV} \left( \mu | A, \nu | B \right) + \textsf{TV} \left( \nu | B, \nu \right),
    \end{split}
\end{equation}
and
\begin{equation}
    \label{eqn:proof_lemma:tv_distance_modification_v1_v2}
    \begin{split}
        \textsf{TV} \left( \mu | A, \nu | B \right) \leq \textsf{TV} \left( \mu | A, \mu \right) + \textsf{TV} \left( \mu, \nu \right) + \textsf{TV} \left( \nu, \nu | B \right).
    \end{split}
\end{equation}
At this point, one can easily observe that
\begin{equation}
    \label{eqn:proof_lemma:tv_distance_modification_v1_v3}
    \begin{split}
        \textsf{TV} \left( \mu, \mu | A \right) = \ & \sup \left\{ \left| \mu (E) - \left( \mu | A \right) (E) \right|: E \in \calF \right\} = \left( \mu | A \right) (A) - \mu (A) = 1 - \mu (A); \\
        \textsf{TV} \left( \nu, \nu | B \right) = \ & \sup \left\{ \left| \nu (E) - \left( \nu | B \right) (E) \right|: E \in \calF \right\} = \left( \nu | B \right) (B) - \nu (B) = 1 - \nu (B).
    \end{split}
\end{equation}
Putting the finding \eqref{eqn:proof_lemma:tv_distance_modification_v1_v3} into the inequalities \eqref{eqn:proof_lemma:tv_distance_modification_v1_v1} and \eqref{eqn:proof_lemma:tv_distance_modification_v1_v2}, the assumptions $1 - \mu (A) \leq \epsilon$ and $1 - \nu (B) \leq \epsilon$ establish the desired result.
    
\end{proof}

\section{Proofs and omitted details for Section \ref{sec:AIPW_non_asymptotic_guarantees}}
\label{sec:proofs_sec:AIPW_non_asymptotic_guarantees}

\subsection{Proof of Proposition \ref{prop:mean_variance_perturbed_ipw}}
\label{subsec:proof_prop:mean_variance_perturbed_ipw}

\indent First, one can observe that
\begin{align}
    \label{eqn:proof_prop:mean_variance_perturbed_ipw_v1}
    &\bbE_{\calI^*} \left[ \hat{\tau}_{n}^{f} \left( \bfO_n \right) \right] \nonumber \\
    = \ & \frac{1}{n} \sum_{i=1}^{n} \bbE_{\calI^*} \left[ \bbE_{\calI^*} \left[ \left. \frac{g \left( X_i, A_i \right) Y_i}{\pi_{i}^* \left( X_i, \bfO_{i-1}; A_i \right)} - f_i \left( X_i, \bfO_{i-1}, A_i \right) \right. \right. \right. \nonumber \\
    & \left. \left. \left. + \left\langle f_i \left( X_i, \bfO_{i-1}, \cdot \right), \pi_{i}^* \left( X_i, \bfO_{i-1}; \cdot \right) \right\rangle_{\lambda_{\bbA}} \right| \left( X_i, A_i, \calH_{i-1} \right) \right] \right] \nonumber \\
    = \ & \frac{1}{n} \sum_{i=1}^{n} \bbE_{\calI^*} \left[ \frac{g \left( X_i, A_i \right) \mu^* \left( X_i, A_i \right)}{\pi_{i}^* \left( X_i, \bfO_{i-1}; A_i \right)} - f_i \left( X_i, \bfO_{i-1}, A_i \right) + \left\langle f_i \left( X_i, \bfO_{i-1}, \cdot \right), \pi_{i}^* \left( X_i, \bfO_{i-1}; \cdot \right) \right\rangle_{\lambda_{\bbA}} \right] \nonumber \\
    = \ & \frac{1}{n} \sum_{i=1}^{n} \bbE_{\calI^*} \left[ \bbE_{\calI^*} \left[ \left. \frac{g \left( X_i, A_i \right) \mu^* \left( X_i, A_i \right)}{\pi_{i}^* \left( X_i, \bfO_{i-1}; A_i \right)} - f_i \left( X_i, \bfO_{i-1}, A_i \right) \right. \right. \right. \\
    & \left. \left. \left. + \left\langle f_i \left( X_i, \bfO_{i-1}, \cdot \right), \pi_{i}^* \left( X_i, \bfO_{i-1}; \cdot \right) \right\rangle_{\lambda_{\bbA}} \right| \left( X_i, \calH_{i-1} \right) \right] \right] \nonumber \\
    = \ & \frac{1}{n} \sum_{i=1}^{n} \bbE_{\calI^*} \left[ \int_{\bbA} g \left( X_i, a \right) \mu^* \left( X_i, a \right) \mathrm{d} \lambda_{\bbA} (a) - \left\langle f_i \left( X_i, \bfO_{i-1}, \cdot \right), \pi_{i}^* \left( X_i, \bfO_{i-1}; \cdot \right) \right\rangle_{\lambda_{\bbA}} \right. \nonumber \\
    & \left. + \left\langle f_i \left( X_i, \bfO_{i-1}, \cdot \right), \pi_{i}^* \left( X_i, \bfO_{i-1}; \cdot \right) \right\rangle_{\lambda_{\bbA}} \right] \nonumber \\
    = \ & \tau \left( \calI^* \right). \nonumber
\end{align}

\indent We now assume \eqref{eqn:prop:mean_variance_perturbed_ipw_v1} and note that
\begin{equation}
    \label{eqn:proof_prop:mean_variance_perturbed_ipw_v2}
    \begin{split}
        &\textsf{Var}_{\calI^*} \left[ \hat{\tau}_{n}^{f} \left( \bfO_n \right) \right] \\ 
        = \ & \frac{1}{n^2} \sum_{i=1}^{n} \textsf{Var}_{\calI^*} \left[ \frac{g \left( X_i, A_i \right) Y_i}{\pi_{i}^* \left( X_i, \bfO_{i-1}; A_i \right)} - f_i \left( X_i, \bfO_{i-1}, A_i \right) \right] \\
        &+ \frac{2}{n^2} \sum_{1 \leq i < j \leq n} \textsf{Cov}_{\calI^*} \left[ \frac{g \left( X_i, A_i \right) Y_i}{\pi_{i}^* \left( X_i, \bfO_{i-1}; A_i \right)} - f_i \left( X_i, \bfO_{i-1}, A_i \right), \frac{g \left( X_j, A_j \right) Y_j}{\pi_{j}^* \left( X_j, \bfO_{j-1}; A_j \right)} - f_j \left( X_j, \bfO_{j-1}, A_j \right) \right].
    \end{split}
\end{equation}
One can reveal that
\begin{align}
    \label{eqn:proof_prop:mean_variance_perturbed_ipw_v3}
    &\textsf{Var}_{\calI^*} \left[ \frac{g \left( X_i, A_i \right) Y_i}{\pi_{i}^* \left( X_i, \bfO_{i-1}; A_i \right)} - f_i \left( X_i, \bfO_{i-1}, A_i \right) \right] \nonumber \\
    = \ & \bbE_{\calI^*} \left[ \bbE_{\calI^*} \left[ \left. \left\{ \frac{g \left( X_i, A_i \right) Y_i}{\pi_{i}^* \left( X_i, \bfO_{i-1}; A_i \right)} - f_i \left( X_i, \bfO_{i-1}, A_i \right) \right\}^2 \right| \left( X_i, A_i, \calH_{i-1} \right) \right] \right] - \left\{ \tau \left( \calI^* \right) \right\}^2 \nonumber \\
    = \ & \bbE_{\calI^*} \left[ \frac{g^2 \left( X_i, A_i \right)}{\left( \pi_{i}^* \right)^2 \left( X_i, \bfO_{i-1}; A_i \right)} \bbE_{\calI^*} \left[ \left. Y_{i}^2 \right| \left( X_i, A_i, \calH_{i-1} \right) \right] \right. \nonumber \\
    & \left. - \frac{2 f_i \left( X_i, \bfO_{i-1}, A_i \right) g \left( X_i, A_i \right)}{\pi_{i}^* \left( X_i, \bfO_{i-1}; A_i \right)} \bbE_{\calI^*} \left[ \left. Y_{i} \right| \left( X_i, A_i, \calH_{i-1} \right) \right] + f_{i}^2 \left( X_i, \bfO_{i-1}, A_i \right) \right] - \left\{ \tau \left( \calI^* \right) \right\}^2 \nonumber \\
    = \ & \bbE_{\calI^*} \left[ \frac{g^2 \left( X_i, A_i \right) \sigma^2 \left( X_i, A_i \right)}{\left( \pi_{i}^* \right)^2 \left( X_i, \bfO_{i-1}; A_i \right)} \right] \nonumber \\
    &+ \bbE_{\calI^*} \left[ \left\{ \frac{g \left( X_i, A_i \right) \mu^* \left( X_i, A_i \right)}{\pi_{i}^* \left( X_i, \bfO_{i-1}; A_i \right)} - f_i \left( X_i, \bfO_{i-1}, A_i \right) \right\}^2 \right] - \left\{ \tau \left( \calI^* \right) \right\}^2 \\
    \stackrel{\textnormal{(a)}}{=} \ & \bbE_{\calI^*} \left[ \frac{g^2 \left( X_i, A_i \right) \sigma^2 \left( X_i, A_i \right)}{\left( \pi_{i}^* \right)^2 \left( X_i, \bfO_{i-1}; A_i \right)} \right] \nonumber \\
    &+ \bbE_{\calI^*} \left[ \left\{ \frac{g \left( X_i, A_i \right) \mu^* \left( X_i, A_i \right)}{\pi_{i}^* \left( X_i, \bfO_{i-1}; A_i \right)} - \left\langle g \left( X_i, \cdot \right), \mu^* \left( X_i, \cdot \right) \right\rangle_{\lambda_{\bbA}} - f_i \left( X_i, \bfO_{i-1}, A_i \right) \right\}^2 \right] \nonumber \\
    &+ \underbrace{\bbE_{\calI^*} \left[ \left\langle g \left( X_i, \cdot \right), \mu^* \left( X_i, \cdot \right) \right\rangle_{\lambda}^2 \right] - \left\{ \tau \left( \calI^* \right) \right\}^2}_{= \ \textsf{Var}_{X \sim \Xi^*} \left[ \left\langle g \left( X, \cdot \right), \mu^* \left( X, \cdot \right) \right\rangle_{\lambda_{\bbA}} \right]} \nonumber \\
    = \ & \bbE_{\calI^*} \left[ \frac{g^2 \left( X_i, A_i \right) \sigma^2 \left( X_i, A_i \right)}{\left( \pi_{i}^* \right)^2 \left( X_i, \bfO_{i-1}; A_i \right)} \right] \nonumber \\
    &+ \bbE_{\calI^*} \left[ \left\{ \frac{g \left( X_i, A_i \right) \mu^* \left( X_i, A_i \right)}{\pi_{i}^* \left( X_i, \bfO_{i-1}; A_i \right)} - \left\langle g \left( X_i, \cdot \right), \mu^* \left( X_i, \cdot \right) \right\rangle_{\lambda_{\bbA}} - f_i \left( X_i, \bfO_{i-1}, A_i \right) \right\}^2 \right] \nonumber \\
    &+ \textsf{Var}_{X \sim \Xi^*} \left[ \left\langle g \left( X, \cdot \right), \mu^* \left( X, \cdot \right) \right\rangle_{\lambda_{\bbA}} \right], \nonumber
\end{align}
where the step (a) can be verified as follows:
\begin{align*}
    &\bbE_{\calI^*} \left[ \left\{ \frac{g \left( X_i, A_i \right) \mu^* \left( X_i, A_i \right)}{\pi_{i}^* \left( X_i, \bfO_{i-1}; A_i \right)} - f_i \left( X_i, \bfO_{i-1}, A_i \right) \right\}^2 \right] \\
    = \ & \bbE_{\calI^*} \left[ \bbE_{\calI^*} \left[ \left. \left\{ \frac{g \left( X_i, A_i \right) \mu^* \left( X_i, A_i \right)}{\pi_{i}^* \left( X_i, \bfO_{i-1}; A_i \right)} - f_i \left( X_i, \bfO_{i-1}, A_i \right) \right\}^2 \right| \left( X_i, \calH_{i-1} \right) \right] \right] \\
    = \ & \bbE_{\calI^*} \left[ \textsf{Var}_{\calI^*} \left[ \left. \frac{g \left( X_i, A_i \right) \mu^* \left( X_i, A_i \right)}{\pi_{i}^* \left( X_i, \bfO_{i-1}; A_i \right)} - f_i \left( X_i, \bfO_{i-1}, A_i \right) \right| \left( X_i, \calH_{i-1} \right) \right] \right] \\
    &+ \bbE_{\calI^*} \left[ \left( \underbrace{\bbE_{\calI^*} \left[ \left. \frac{g \left( X_i, A_i \right) \mu^* \left( X_i, A_i \right)}{\pi_{i}^* \left( X_i, \bfO_{i-1}; A_i \right)} - f_i \left( X_i, \bfO_{i-1}, A_i \right) \right| \left( X_i, \calH_{i-1} \right) \right]}_{= \ \left\langle g \left( X_i, \cdot \right), \mu^* \left( X_i, \cdot \right) \right\rangle_{\lambda}} \right)^2 \right] \\
    = \ & \bbE_{\calI^*} \left[ \left\{ \frac{g \left( X_i, A_i \right) \mu^* \left( X_i, A_i \right)}{\pi_{i}^* \left( X_i, \bfO_{i-1}; A_i \right)} - \left\langle g \left( X_i, \cdot \right), \mu^* \left( X_i, \cdot \right) \right\rangle_{\lambda_{\bbA}} - f_i \left( X_i, \bfO_{i-1}, A_i \right) \right\}^2 \right] \\
    &+ \bbE_{\calI^*} \left[ \left\langle g \left( X_i, \cdot \right), \mu^* \left( X_i, \cdot \right) \right\rangle_{\lambda_{\bbA}}^2 \right].
\end{align*}

\indent Next, we compute $\textsf{Cov}_{\calI^*} \left[ \frac{g \left( X_i, A_i \right) Y_i}{\pi_{i}^* \left( X_i, \bfO_{i-1}; A_i \right)} - f_i \left( X_i, \bfO_{i-1}, A_i \right), \frac{g \left( X_j, A_j \right) Y_j}{\pi_{j}^* \left( X_j, \bfO_{j-1}; A_j \right)} - f_j \left( X_j, \bfO_{j-1}, A_j \right) \right]$:

\begin{equation}
    \label{eqn:proof_prop:mean_variance_perturbed_ipw_v4}
    \begin{split}
        &\textsf{Cov}_{\calI^*} \left[ \frac{g \left( X_i, A_i \right) Y_i}{\pi_{i}^* \left( X_i, \bfO_{i-1}; A_i \right)} - f_i \left( X_i, \bfO_{i-1}, A_i \right), \frac{g \left( X_j, A_j \right) Y_j}{\pi_{j}^* \left( X_j, \bfO_{j-1}; A_j \right)} - f_j \left( X_j, \bfO_{j-1}, A_j \right) \right] \\ 
        = \ & \bbE_{\calI^*} \left[ \left\{ \frac{g \left( X_i, A_i \right) Y_i}{\pi_{i}^* \left( X_i, \bfO_{i-1}; A_i \right)} - f_i \left( X_i, \bfO_{i-1}, A_i \right) \right\} \right. \\
        &\left. \left\{ \frac{g \left( X_j, A_j \right) \mu^* \left( X_j, A_j \right)}{\pi_{j}^* \left( X_j, \bfO_{j-1}; A_j \right)} - f_j \left( X_j, \bfO_{j-1}, A_j \right) \right\} \right] - \left\{ \tau \left( \calI^* \right) \right\}^2 \\
        = \ & \bbE_{\calI^*} \left[ \bbE_{\calI^*} \left[ \left\{ \frac{g \left( X_i, A_i \right) Y_i}{\pi_{i}^* \left( X_i, \bfO_{i-1}; A_i \right)} - f_i \left( X_i, \bfO_{i-1}, A_i \right) \right\} \right. \right. \\
        &\left. \left. \left. \left\{ \frac{g \left( X_j, A_j \right) \mu^* \left( X_j, A_j \right)}{\pi_{j}^* \left( X_j, \bfO_{j-1}; A_j \right)} - f_j \left( X_j, \bfO_{j-1}, A_j \right) \right\} \right| \left( X_j, \calH_{j-1} \right) \right] \right] - \left\{ \tau \left( \calI^* \right) \right\}^2 \\
        = \ & \bbE_{\calI^*} \left[ \left\{ \frac{g \left( X_i, A_i \right) Y_i}{\pi_{i}^* \left( X_i, \bfO_{i-1}; A_i \right)} - f_i \left( X_i, \bfO_{i-1}, A_i \right) \right\} \left\langle g \left( X_j, \cdot \right), \mu^* \left( X_j, \cdot \right) \right\rangle_{\lambda_{\bbA}} \right] - \left\{ \tau \left( \calI^* \right) \right\}^2 \\
        = \ & \bbE_{\calI^*} \left[ \bbE_{\calI^*} \left[ \left. \left\{ \frac{g \left( X_i, A_i \right) Y_i}{\pi_{i}^* \left( X_i, \bfO_{i-1}; A_i \right)} - f_i \left( X_i, \bfO_{i-1}, A_i \right) \right\} \left\langle g \left( X_j, \cdot \right), \mu^* \left( X_j, \cdot \right) \right\rangle_{\lambda_{\bbA}} \right| \calH_{j-1} \right] \right] - \left\{ \tau \left( \calI^* \right) \right\}^2 \\
        \stackrel{\textnormal{(b)}}{=} \ & 0,
    \end{split} 
\end{equation}
where the step (b) holds due to the fact that $X_j$ is independent of the historical data $\calH_{j-1}$, which immediately yields $\left. X_j \right| \calH_{j-1} \stackrel{d}{=} X_j \sim \Xi^* (\cdot)$. By taking two equations \eqref{eqn:proof_prop:mean_variance_perturbed_ipw_v3} and \eqref{eqn:proof_prop:mean_variance_perturbed_ipw_v4} collectively into the equation \eqref{eqn:proof_prop:mean_variance_perturbed_ipw_v2}, one has
\[
    \begin{split}
        &n \cdot \textsf{Var}_{\calI^*} \left[ \hat{\tau}_{n}^{f} \left( \bfO_n \right) \right] \\
        = \ & \textsf{Var}_{X \sim \Xi^*} \left[ \left\langle g ( X, \cdot ), \mu^* ( X, \cdot ) \right\rangle_{\lambda_{\bbA}} \right] \\
        &+ \frac{1}{n} \sum_{i=1}^{n} \left( \bbE_{\calI^*} \left[ \frac{g^2 \left( X_i, A_i \right) \sigma^2 \left( X_i, A_i \right)}{\left( \pi_{i}^* \right)^2 \left( X_i, \bfO_{i-1}, A_i \right)} \right] \right. \\
        &\left. + \bbE_{\calI^*} \left[ \left\{ \frac{g \left( X_i, A_i \right) \mu^* \left( X_i, A_i \right)}{\pi_{i}^* \left( X_i, \bfO_{i-1}; A_i \right)} - \left\langle g \left( X_i, \cdot \right), \mu^* \left( X_i, \cdot \right) \right\rangle_{\lambda_{\bbA}} - f_i \left( X_i, \bfO_{i-1}, A_i \right) \right\}^2 \right] \right),
    \end{split}
\]
as desired.

\subsection{Proof of Theorem \ref{thm:mse_aipw}}
\label{subsec:proof_thm:mse_aipw}

We first single out a key technical lemma throughout this section that plays a significant role in the proof of Theorem \ref{thm:mse_aipw}. 

\begin{lemma}
\label{lemma:aipw_v1}
The following results hold:
\begin{enumerate} [label = (\roman*)]
    \item It holds that $\bbE_{\calI^*} \left[ \left. \hat{\Gamma}_i \left( \bfO_i \right) \right| \left( X_i, \calH_{i-1} \right) \right] = \left\langle g \left( X_i, \cdot \right), \mu^* \left( X_i, \cdot \right) \right\rangle_{\lambda_{\bbA}}$ for all $i \in [n]$. Therefore, one has
    \begin{equation}
        \label{eqn:lemma:aipw_v1_v1}
        \begin{split}
            \bbE_{\calI^*} \left[ \hat{\Gamma}_i \left( \bfO_i \right) \right] = \bbE_{\calI^*} \left[ \bbE_{\calI^*} \left[ \left. \hat{\Gamma}_i \left( \bfO_i \right) \right| \left( X_i, \calH_{i-1} \right) \right] \right] = \bbE_{\calI^*} \left[ \left\langle g \left( X_i, \cdot \right), \mu^* \left( X_i, \cdot \right) \right\rangle_{\lambda_{\bbA}} \right] = \tau \left( \calI^* \right).
        \end{split}
    \end{equation}
    \item For every $1 \leq i < j \leq n$, we have $\textnormal{\textsf{Cov}}_{\calI^*} \left[ \hat{\Gamma}_i \left( \bfO_i \right), \hat{\Gamma}_j \left( \bfO_j \right) \right] = 0$;
    \item For every $i \in [n]$,
    \begin{equation}
        \label{eqn:lemma:aipw_v1_v2}
        \begin{split}
            &\textnormal{\textsf{Var}}_{\calI^*} \left[ \hat{\Gamma}_i \left( \bfO_i \right) \right] \\
            = \ & \textnormal{\textsf{Var}}_{X \sim \Xi^*} \left[ \left\langle g \left( X, \cdot \right), \mu^* \left( X, \cdot \right) \right\rangle_{\lambda_{\bbA}} \right] + \bbE_{\calI^*} \left[ \frac{g^2 \left( X_i, A_i \right) \sigma^2 \left( X_i, A_i \right)}{\left( \pi_{i}^* \right)^2 \left( X_i, \bfO_{i-1}; A_i \right)} \right] \\
            &+ \bbE_{\calI^*} \left[ \textnormal{\textsf{Var}}_{\calI^*} \left[ \left. \frac{g \left( X_i, A_i \right)}{\pi_{i}^* \left( X_i, \bfO_{i-1}; A_i \right)} \left\{ \hat{\mu}_{i} \left( \bfO_{i-1} \right) \left( X_i, A_i \right) - \mu^* \left( X_i, A_i \right) \right\} \right| \left( X_i, \calH_{i-1} \right) \right] \right] \\
            \leq \ & \textnormal{\textsf{Var}}_{X \sim \Xi^*} \left[ \left\langle g \left( X, \cdot \right), \mu^* \left( X, \cdot \right) \right\rangle_{\lambda_{\bbA}} \right] + \bbE_{\calI^*} \left[ \frac{g^2 \left( X_i, A_i \right) \sigma^2 \left( X_i, A_i \right)}{\left( \pi_{i}^* \right)^2 \left( X_i, \bfO_{i-1}; A_i \right)} \right] \\
            &+ \bbE_{\calI^*} \left[ \frac{g^2 \left( X_i, A_i \right) \left\{ \hat{\mu}_{i} \left( \bfO_{i-1} \right) \left( X_i, A_i \right) - \mu^* \left( X_i, A_i \right) \right\}^2}{\left( \pi_{i}^* \right)^2 \left( X_i, \bfO_{i-1}; A_i \right)} \right].
        \end{split}
    \end{equation}
\end{enumerate}
\end{lemma}

\begin{proof} [Proof of Lemma \ref{lemma:aipw_v1}] \ \\
\indent (\romannumeral 1) From the definition of $\hat{\Gamma}_i (\cdot): \bbO^i \to \bbR$ in \eqref{eqn:alg:aipw_estimator_v2}, we have
\begin{equation}
    \label{eqn:proof_lemma:aipw_v1_v1}
    \begin{split}
        \bbE_{\calI^*} \left[ \left. \hat{\Gamma}_i \left( \bfO_i \right) \right| \left( X_i, A_i, \calH_{i-1} \right) \right] 
        = \ & \frac{g \left( X_i, A_i \right)}{\pi_{i}^* \left( X_i, \bfO_{i-1}; A_i \right)} \left\{ \mu^* \left( X_i, A_i \right) - \hat{\mu}_{i} \left( \bfO_{i-1} \right) \left( X_i, A_i \right) \right\} \\
        &+ \left\langle g \left( X_i, \cdot \right), \hat{\mu}_{i} \left( \bfO_{i-1} \right) \left( X_i, \cdot \right) \right\rangle_{\lambda_{\bbA}}.
    \end{split}
\end{equation}
Thus, we obtain
\begin{equation}
    \label{eqn:proof_lemma:aipw_v1_v2}
    \begin{split}
        &\bbE_{\calI^*} \left[ \left. \hat{\Gamma}_i \left( \bfO_i \right) \right| \left( X_i, \calH_{i-1} \right) \right] \\
        = \ & \bbE_{\calI^*} \left[ \left. \bbE_{\calI^*} \left[ \left. \hat{\Gamma}_i \left( \bfO_i \right) \right| \left( X_i, A_i, \calH_{i-1} \right) \right] \right| \left( X_i, \calH_{i-1} \right) \right] \\
        = \ & \int_{\bbA} \frac{g \left( X_i, a \right)}{\pi_{i}^* \left( X_i, \bfO_{i-1}; a \right)} \left\{ \mu^* \left( X_i, a \right) - \hat{\mu}_i \left( \bfO_{i-1} \right) \left( X_i, a \right) \right\} \cdot \pi_{i}^* \left( X_i, \bfO_{i-1}; a \right) \mathrm{d} \lambda_{\bbA} (a) \\
        &+ \left\langle g \left( X_i, \cdot \right), \hat{\mu}_{i} \left( \bfO_{i-1} \right) \left( X_i, \cdot \right) \right\rangle_{\lambda_{\bbA}} \\
        = \ & \left\langle g \left( X_i, \cdot \right), \mu^* \left( X_i, \cdot \right) \right\rangle_{\lambda_{\bbA}}
    \end{split}
\end{equation}
as desired.
\medskip

\indent (\romannumeral 2) One can reveal that
\begin{equation}
    \label{eqn:proof_lemma:aipw_v1_v3}
    \begin{split}
        &\textsf{Cov}_{\calI^*} \left[ \hat{\Gamma}_i \left( \bfO_i \right), \hat{\Gamma}_j \left( \bfO_j \right) \right] \\
        = \ & \bbE_{\calI^*} \left[ \hat{\Gamma}_i \left( \bfO_i \right) \bbE \left[ \left. \hat{\Gamma}_j \left( \bfO_j \right) \right| \left( X_j, A_j, \calH_{j-1} \right) \right] \right] - \left\{ \tau \left( \calI^* \right) \right\}^2 \\
        = \ & \bbE_{\calI^*} \left[ \hat{\Gamma}_i \left( \bfO_i \right) \left[ \frac{g \left( X_j, A_j \right)}{\pi_{j}^* \left( X_j, \bfO_{j-1}; A_j \right)} \left\{ \mu^* \left( X_j, A_j \right) - \hat{\mu}_{j} \left( \bfO_{j-1} \right) \left( X_j, A_j \right) \right\} \right. \right. \\
        &\left. \left. + \left\langle g \left( X_j, \cdot \right), \hat{\mu}_{j} \left( \bfO_{j-1} \right) \left( X_j, \cdot \right) \right\rangle_{\lambda_{\bbA}} \right] \right] - \left\{ \tau \left( \calI^* \right) \right\}^2 \\
        = \ & \bbE_{\calI^*} \left[ \hat{\Gamma}_i \left( \bfO_i \right) \bbE_{\calI^*} \left[ \frac{g \left( X_j, A_j \right)}{\pi_{j}^* \left( X_j, \bfO_{j-1}; A_j \right)} \left\{ \mu^* \left( X_j, A_j \right) - \hat{\mu}_{j} \left( \bfO_{j-1} \right) \left( X_j, A_j \right) \right\} \right. \right. \\
        &\left. \left. \left.+ \left\langle g \left( X_j, \cdot \right), \hat{\mu}_{j} \left( \bfO_{j-1} \right) \left( X_j, \cdot \right) \right\rangle_{\lambda_{\bbA}} \right| \left( X_j, \calH_{j-1} \right) \right] \right] - \left\{ \tau \left( \calI^* \right) \right\}^2 \\
        = \ & \bbE_{\calI^*} \left[ \hat{\Gamma}_i \left( \bfO_i; g \right) \left\langle g \left( X_j, \cdot \right), \mu^* \left( X_j, \cdot \right) \right\rangle_{\lambda_{\bbA}} \right] - \left\{ \tau \left( \calI^* \right) \right\}^2 \\
        \stackrel{\textnormal{(a)}}{=} \ & 0,
    \end{split}
\end{equation}
where the step (a) holds due to the facts that $\hat{\Gamma}_i \left( \bfO_i \right)$ is $\calH_{j-1}$-measurable and $X_j \indep \calH_{j-1}$, together with the equation \eqref{eqn:lemma:aipw_v1_v1}.
\medskip

\indent (\romannumeral 3) It follows that
\begin{align}
    \label{eqn:proof_lemma:aipw_v1_v4}
    &\textsf{Var}_{\calI^*} \left[ \hat{\Gamma}_i \left( \bfO_i \right) \right] \nonumber \\
    = \ & \bbE_{\calI^*} \left[ \textsf{Var}_{\calI^*} \left[ \left. \hat{\Gamma}_i \left( \bfO_i \right) \right| \left( X_i, \calH_{i-1} \right) \right] \right] + \textsf{Var}_{\calI^*} \left[ \bbE_{\calI^*} \left[ \left. \hat{\Gamma}_i \left( \bfO_i \right) \right| \left( X_i, \calH_{i-1} \right) \right] \right] \nonumber \\
    \stackrel{\textnormal{(b)}}{=} \ & \bbE_{\calI^*} \left[  \bbE_{\calI^*} \left[ \left. \textsf{Var}_{\calI^*} \left[ \left. \hat{\Gamma}_i \left( \bfO_i \right) \right| \left( X_i, A_i, \calH_{i-1} \right) \right] \right| \left( X_i, \calH_{i-1} \right) \right] \right] \nonumber \\
    &+ \bbE_{\calI^*} \left[ \textsf{Var}_{\calI^*} \left[ \left. \bbE_{\calI^*} \left[ \left. \hat{\Gamma}_i \left( \bfO_i \right) \right| \left( X_i, A_i, \calH_{i-1} \right) \right] \right| \left( X_i, \calH_{i-1} \right) \right] \right] \\
    &+ \textsf{Var}_{X \sim \Xi^*} \left[ \left\langle g (X, \cdot), \mu^* (X, \cdot) \right\rangle_{\lambda_{\bbA}} \right] \nonumber \\
    = \ & \bbE_{\calI^*} \left[ \frac{g^2 \left( X_i, A_i \right) \sigma^2 \left( X_i, A_i \right)}{\left( \pi_{i}^* \right)^2 \left( X_i, \bfO_{i-1}; A_i \right)} \right] \nonumber \\
    &+ \bbE_{\calI^*} \left[ \textsf{Var}_{\calI^*} \left[ \left. \frac{g \left( X_i, A_i \right)}{\pi_{i}^* \left( X_i, \bfO_{i-1}; A_i \right)} \left\{ \mu^* \left( X_i, A_i \right) - \hat{\mu}_{i} \left( \bfO_{i-1} \right) \left( X_i, A_i \right) \right\} \right| \left( X_i, \calH_{i-1} \right) \right] \right] \nonumber \\
    &+ \textsf{Var}_{X \sim \Xi^*} \left[ \left\langle g (X, \cdot), \mu^* (X, \cdot) \right\rangle_{\lambda_{\bbA}} \right], \nonumber
\end{align}
as desired, where the step (b) follows from the fact \eqref{eqn:proof_lemma:aipw_v1_v2}.

\end{proof}

Now, it's time to finish the proof of Theorem \ref{thm:mse_aipw}. One can reveal that
\begin{align*}
    &\bbE_{\calI^*} \left[ \left\{ \hat{\tau}_{n}^{\textnormal{\textsf{AIPW}}} \left( \bfO_n; g \right) - \tau \left( \calI^*; g \right) \right\}^2 \right] \\
    \stackrel{\textnormal{(a)}}{=} \ & \frac{1}{n^2} \sum_{i=1}^{n} \textsf{Var}_{\calI^*} \left[ \hat{\Gamma}_i \left( \bfO_i; g \right) \right] \\
    \stackrel{\textnormal{(b)}}{\leq} \ & \frac{1}{n^2} \sum_{i=1}^{n} \left\{ \textsf{Var}_{X \sim \Xi^*} \left[ \left\langle g \left( X, \cdot \right), \mu^* \left( X, \cdot \right) \right\rangle_{\lambda_{\bbA}} \right] + \bbE_{\calI^*} \left[ \frac{g^2 \left( X_i, A_i \right) \sigma^2 \left( X_i, A_i \right)}{\left( \pi_{i}^* \right)^2 \left( X_i, \bfO_{i-1}; A_i \right)} \right] \right. \\
    &\left. + \bbE_{\calI^*} \left[ \frac{g^2 \left( X_i, A_i \right) \left\{ \mu^* \left( X_i, A_i \right) - \hat{\mu}_{i} \left( \bfO_{i-1} \right) \left( X_i, A_i \right) \right\}^2}{\left( \pi_{i}^* \right)^2 \left( X_i, \bfO_{i-1}; A_i \right)} \right] \right\} \\
    \stackrel{\textnormal{(c)}}{=} \ & \frac{1}{n} \left\{ v_{*}^2 + \frac{1}{n} \sum_{i=1}^{n} \bbE_{\calI^*} \left[ \frac{g^2 \left( X_i, A_i \right) \left\{ \hat{\mu}_{i} \left( \bfO_{i-1} \right) \left( X_i, A_i \right) - \mu^* \left( X_i, A_i \right) \right\}^2}{\left( \pi_{i}^* \right)^2 \left( X_i, \bfO_{i-1}; A_i \right)} \right] \right\},
\end{align*}
where the step (a) holds due to the part (\romannumeral 2) of Lemma \ref{lemma:aipw_v1}, the step (b) makes use of the inequality \eqref{eqn:lemma:aipw_v1_v2}, and the step (c) follows from the definition of $v_{*}^2$ in \eqref{eqn:optimal_variance}.

\subsection{Proof of Theorem \ref{thm:online_regression_oracle_ineq_v1}}
\label{subsec:proof_thm:online_regression_oracle_ineq_v1}

\indent It holds due to the observation \eqref{eqn:property_online_regression_v1_loss} that
\begin{align}
    &\bbE_{\calI^*} \left[ \sum_{i=1}^{n} l_i \left\{ \hat{\mu}_i \left( \bfO_{i-1} \right) \right\} \right] \nonumber \\
    = \ & \sum_{i=1}^{n} \bbE_{\calI^*} \left[ \bbE_{\calI^*} \left[ \left. l_i \left\{ \hat{\mu}_i \left( \bfO_{i-1} \right) \right\} \right| \left( \calH_{i-1}, X_i, A_i \right) \right] \right] \nonumber \\
    = \ & \sum_{i=1}^{n} \bbE_{\calI^*} \left[ \frac{g^2 \left( X_i, A_i \right)}{\left( \pi_{i}^* \right)^2 \left( X_i, \bfO_{i-1}; A_i \right)} \left[ \sigma^2 \left( X_i, A_i \right) + \left\{ \hat{\mu}_i \left( \bfO_{i-1} \right) \left( X_i, A_i \right) - \mu^* \left( X_i, A_i \right) \right\}^2 \right] \right] \nonumber \\
    = \ & n \left\| \sigma \right\|_{(n)}^2 + \sum_{i=1}^{n} \bbE_{\calI^*} \left[  \frac{g^2 \left( X_i, A_i \right) \left\{ \hat{\mu}_i \left( \bfO_{i-1} \right) \left( X_i, A_i \right) - \mu^* \left( X_i, A_i \right) \right\}^2}{\left( \pi_{i}^* \right)^2 \left( X_i, \bfO_{i-1}; A_i \right)} \right] \nonumber,
\end{align}
which establishes the following expression of the estimation error term \eqref{eqn:treatment_effect_estimation_error}:
\begin{align}
    \label{eqn:proof_thm:online_regression_oracle_ineq_v1_v1}
        &\frac{1}{n} \sum_{i=1}^{n} \bbE_{\calI^*} \left[ \frac{g^2 \left( X_i, A_i \right) \left\{ \hat{\mu}_{i} \left( \bfO_{i-1} \right) \left( X_i, A_i \right) - \mu^* \left( X_i, A_i \right) \right\}^2}{\left( \pi_{i}^* \right)^2 \left( X_i, \bfO_{i-1}; A_i \right)} \right] \nonumber \\
        = \ & \frac{1}{n} \bbE_{\calI^*} \left[ \sum_{i=1}^{n} l_i \left\{ \hat{\mu}_i \left( \bfO_{i-1} \right) \right\} \right] - \left\| \sigma \right\|_{(n)}^2 \\
        = \ & \frac{1}{n} \bbE_{\calI^*} \left[ \textsf{Regret} \left( n; \calA \right) \right] + \frac{1}{n} \bbE_{\calI^*} \left[ \inf \left\{ \sum_{i=1}^{n} l_i (\mu): \mu \in \calF \right\} \right] - \left\| \sigma \right\|_{(n)}^2 \nonumber.
\end{align}
At this point, one can realize that
\begin{align}
    \label{eqn:proof_thm:online_regression_oracle_ineq_v1_v2}
        &\frac{1}{n} \bbE_{\calI^*} \left[ \inf \left\{ \sum_{i=1}^{n} l_i (\mu): \mu \in \calF \right\} \right] \nonumber \\
        \leq \ & \inf \left\{ \frac{1}{n} \bbE_{\calI^*} \left[ \sum_{i=1}^{n} l_i (\mu) \right]: \mu \in \calF \right\} \nonumber \\
        = \ & \inf \left\{ \frac{1}{n} \sum_{i=1}^{n} \bbE_{\calI^*} \left[ \bbE_{\calI^*} \left[ \left. l_i (\mu) \right| \left( \calH_{i-1}, X_i, A_i \right) \right] \right]: \mu \in \calF \right\} \\
        \stackrel{\textnormal{(a)}}{=} \ & \inf \left\{ \frac{1}{n} \sum_{i=1}^{n} \bbE_{\calI^*} \left[ \frac{g^2 \left( X_i, A_i \right)}{\left( \pi_{i}^* \right)^2 \left( X_i, \bfO_{i-1}; A_i \right)} \left[ \sigma^2 \left( X_i, A_i \right) + \left\{ \mu \left( X_i, A_i \right) - \mu^* \left( X_i, A_i \right) \right\}^2 \right] \right]: \mu \in \calF \right\} \nonumber \\
        = \ & \left\| \sigma \right\|_{(n)}^2 + \inf \left\{ \left\| \mu - \mu^* \right\|_{(n)}^2 : \mu \in \calF \right\} \nonumber,
\end{align}
where the step (a) holds by the fact \eqref{eqn:property_online_regression_v1_loss}. Taking two pieces \eqref{eqn:proof_thm:online_regression_oracle_ineq_v1_v1} and \eqref{eqn:proof_thm:online_regression_oracle_ineq_v1_v2} collectively, it follows that
\begin{equation}
    \label{eqn:proof_thm:online_regression_oracle_ineq_v1_v3}
    \begin{split}
         &\frac{1}{n} \sum_{i=1}^{n} \bbE_{\calI^*} \left[ \frac{g^2 \left( X_i, A_i \right) \left\{ \hat{\mu}_{i} \left( \bfO_{i-1} \right) \left( X_i, A_i \right) - \mu^* \left( X_i, A_i \right) \right\}^2}{\left( \pi_{i}^* \right)^2 \left( X_i, \bfO_{i-1}; A_i \right)} \right] \\
         \leq \ & \frac{1}{n} \bbE_{\calI^*} \left[ \textsf{Regret} \left( n; \calA \right) \right] + \inf \left\{ \left\| \mu - \mu^* \right\|_{(n)}^2 : \mu \in \calF \right\}.
    \end{split}
\end{equation}
Hence, the upper bound \eqref{eqn:thm:online_regression_oracle_ineq_v1_v1} on the \textsf{MSE} of the \textsf{AIPW} estimator \eqref{eqn:alg:aipw_estimator_v1} is an immediate consequence of the inequality \eqref{eqn:proof_thm:online_regression_oracle_ineq_v1_v3} by putting it into the bound \eqref{eqn:thm:mse_aipw_v1} in Theorem \ref{thm:mse_aipw}.

\subsection{Proof of Theorem \ref{thm:regret_bound_ogd_v1}}
\label{subsec:proof_thm:regret_bound_ogd_v1}

One can easily observe from the equation \eqref{eqn:gradient_online_regression_v1_loss} for every $\mu \in \calF$ that
\begin{equation}
    \label{eqn:proof_thm:regret_bound_ogd_v1_v1}
    \begin{split}
        \left\| \nabla l_i (\mu) \right\|_{2}^2 = \frac{4 g^4 \left( X_i, A_i \right)}{\left( \pi_{i}^* \right)^4 \left( X_i, \bfO_{i-1}; A_i \right)} \left\{ Y_i - \mu \left( X_i, A_i \right) \right\}^2 \stackrel{\bbP_{\calI^*}^n\textnormal{-a.s.}}{\leq} \left( 4 L B^2 \right)^2,
    \end{split}
\end{equation}
which holds due to Assumption \ref{assumption:strict_overlap_condition} together with the fact $\bbY = \left[ -L, L \right]$. So, it turns out that the loss function \eqref{eqn:online_regression_v1_loss} is Lipschitz continuous with parameter $G := 4LB^2$ $\bbP_{\calI^*}^n$-almost surely. Hence, the desired conclusion immediately follows by \emph{Theorem 3.1} in \cite{hazan2016introduction} with parameter $G = 4LB^2$. 

\subsection{Proof of Theorem \ref{thm:regret_bound_ogd_v2}}
\label{subsec:proof_thm:regret_bound_ogd_v2}

One can realize from the equation \eqref{eqn:gradient_linear_approx_online_learning_loss} that $\bbP_{\calI^*}^{n}$-almost surely,
\begin{equation}
    \label{eqn:proof_thm:regret_bound_ogd_v2_v1}
    \begin{split}
        \left\| \nabla_{\bftheta} \calL_i (\bftheta) \right\|_{2}^2 = \ & \frac{4 g^4 \left( X_i, A_i \right)}{\left( \pi_{i}^* \right)^4 \left( X_i, \bfO_{i-1}; A_i \right)} \left\{ \bftheta^{\top} \phi \left( X_i, A_i \right) - Y_i \right\}^2 \left\| \phi \left( X_i, A_i \right) \right\|_{2}^2 \\
        \leq \ & 4 B^4 \left\{ \left| Y_i \right| + \left\| \bftheta \right\|_{2} \left\| \phi \left( X_i, A_i \right) \right\|_{2} \right\}^2 \left\| \phi \left( X_i, A_i \right) \right\|_{2}^2 \\
        \leq \ & 4 B^4 ( L + R )^2,
    \end{split}
\end{equation}
which holds due to Assumption \ref{assumption:strict_overlap_condition} together with the facts $\bbY = \left[ -L, L \right]$ and $\sup_{(x, a) \in \bbX \times \bbA} \left\| \phi (x, a) \right\|_{2} \leq 1$. So, the loss function \eqref{eqn:linear_approx_online_learning_loss} is Lipschitz continuous with parameter $G := 2 B^2 (L+R)$ $\bbP_{\calI^*}^n$-a.s. Hence, the desired result follows by \emph{Theorem 3.1} in \cite{hazan2016introduction} with parameter $G = 2 B^2 (L+R)$ and $D = 2R$.

\subsection{Consequences for particular outcome models: general function approximation}
\label{subsec:general_function_approximation}

Finally, it is time to consider the most challenging setting where the estimation of the treatment effect $\mu^*$ is parameterized by general function classes. Under Assumption \ref{assumption:strict_overlap_condition}, one can first observe from the \textsf{MSE} bound \eqref{eqn:thm:mse_aipw_v1} of the \textsf{AIPW} estimator \eqref{eqn:alg:aipw_estimator_v1} in Theorem \ref{thm:mse_aipw} that
\begin{equation}
    \label{eqn:general_function_approximation_v1}
    \begin{split}
        &\bbE_{\calI^*} \left[ \left\{ \hat{\tau}_{n}^{\textnormal{\textsf{AIPW}}} \left( \bfO_n \right) - \tau \left( \calI^* \right) \right\}^2 \right] \\
        \leq \ & \frac{1}{n} \left\{ v_{*}^2 + \frac{1}{n} \sum_{i=1}^{n} \bbE_{\calI^*} \left[ \frac{g^2 \left( X_i, A_i \right) \left\{ \hat{\mu}_{i} \left( \bfO_{i-1} \right) \left( X_i, A_i \right) - \mu^* \left( X_i, A_i \right) \right\}^2}{\left( \pi_{i}^* \right)^2 \left( X_i, \bfO_{i-1}; A_i \right)} \right] \right\} \\
        \leq \ & \frac{1}{n} \left\{ v_{*}^2 + \frac{B^2}{n} \sum_{i=1}^{n} \bbE_{\calI^*} \left[ \left\{ \hat{\mu}_{i} \left( \bfO_{i-1} \right) \left( X_i, A_i \right) - \mu^* \left( X_i, A_i \right) \right\}^2 \right] \right\}.
    \end{split}
\end{equation}
From the last term in the \textsf{MSE} bound \eqref{eqn:general_function_approximation_v1}, our aim becomes to control an upper bound of the term
\begin{equation}
    \label{eqn:treatment_effect_estimation_error_general_func_approx}
    \begin{split}
        \frac{1}{n} \sum_{i=1}^{n} \bbE_{\calI^*} \left[ \left\{ \hat{\mu}_{i} \left( \bfO_{i-1} \right) \left( X_i, A_i \right) - \mu^* \left( X_i, A_i \right) \right\}^2 \right]
    \end{split}
\end{equation}
in the finite sample regime. Towards this end, we consider the online non-parametric regression problem in Algorithm \ref{alg:online_regression_v1} whose sequence of loss functions $\left\{ l_i (\cdot): \left( \bbX \times \bbA \to \bbR \right) \to \bbR : i \in [n] \right\}$ defined as \eqref{eqn:online_regression_v1_loss} is superseded by $\left\{ \overline{l}_i (\cdot): \left( \bbX \times \bbA \to \bbR \right) \to \bbR : i \in [n] \right\}$, where
\begin{equation}
    \label{eqn:online_regression_v1_loss_general_func_approx}
    \begin{split}
        \overline{l}_i (\mu) := \left\{ Y_i - \mu \left( X_i, A_i \right) \right\}^2,\ \forall \left( \mu, i \right) \in \left( \bbX \times \bbA \to \bbR \right) \times [n].
    \end{split}
\end{equation}
It is straightforward to see for every $i \in [n]$ that
\begin{equation}
    \label{eqn:property_online_regression_v1_loss_general_func_approx}
    \begin{split}
        \bbE_{\calI^*} \left[ \left. \overline{l}_i (\mu) \right| \left( \calH_{i-1}, X_i, A_i \right) \right] = \sigma^2 \left( X_i, A_i \right) + \left\{ \mu \left( X_i, A_i \right) - \mu^* \left( X_i, A_i \right) \right\}^2.
    \end{split}
\end{equation}
With this modified online non-parametric regression problem, we now aim to minimize the learner's \emph{modified regret} defined as follows:
\begin{equation}
    \label{eqn:defi_regret_v2}
    \begin{split}
        \overline{\textsf{Regret}} \left( n, \calF; \overline{\calA} \right) := \sum_{i=1}^{n} \overline{l}_i \left\{ \hat{\mu}_i \left( \bfO_{i-1} \right) \right\} - \inf \left\{ \sum_{i=1}^{n} \overline{l}_i (\mu) : \mu \in \calF \right\},
    \end{split}
\end{equation}
where $\overline{\calA}$ denotes the learner's online non-parametric regression algorithm that returns a sequence of estimates $\left\{ \hat{\mu}_i \left( \bfO_{i-1} \right) \in \left( \bbX \times \bbA \to \bbR \right): i \in [n]\right\}$ of the treatment effect based on interactions with the environment which selects modified loss functions $\left\{ \overline{l}_i (\cdot): \left( \bbX \times \bbA \to \bbR \right) \to \bbR : i \in [n] \right\}$.

\begin{thm}
\label{thm:online_regression_oracle_ineq_general_func_approx}
The \textnormal{\textsf{AIPW}} estimator \eqref{eqn:alg:aipw_estimator_v1} based on a sequence $\left\{ \hat{\mu}_i \left( \bfO_{i-1} \right) \in \left( \bbX \times \bbA \to \bbR \right): i \in [n] \right\}$ of estimates for the treatment effect $\mu^*$ produced by making use of an online non-parametric regression algorithm $\overline{\calA}$ against the environment which chooses the sequence of modified loss functions $\left\{ \overline{l}_i (\cdot): \left( \bbX \times \bbA \to \bbR \right) \to \bbR : i \in [n] \right\}$ defined in \eqref{eqn:online_regression_v1_loss_general_func_approx} enjoys the following upper bound on the \textnormal{\textsf{MSE}}:
\begin{equation}
    \label{eqn:thm:online_regression_oracle_ineq_general_func_approx_v1}
    \begin{split}
        &\bbE_{\calI^*} \left[ \left\{ \hat{\tau}_{n}^{\textnormal{\textsf{AIPW}}} \left( \bfO_n \right) - \tau \left( \calI^* \right) \right\}^2 \right] \\
        \leq \ & \frac{1}{n} \left( v_{*}^2 + \frac{1}{n} \bbE_{\calI^*} \left[ \overline{\textnormal{\textsf{Regret}}} \left( n, \calF; \overline{\calA} \right) \right] + \underbrace{\inf \left\{ \frac{1}{n} \sum_{i=1}^{n} \bbE_{\calI^*} \left[ \left\{ \mu \left( X_i, A_i \right) - \mu^* \left( X_i, A_i \right) \right\}^2 \right] : \mu \in \calF \right\}}_{\textnormal{approximation error term.}} \right).
    \end{split}
\end{equation}
\end{thm}

\begin{proof} [Proof of Theorem \ref{thm:online_regression_oracle_ineq_general_func_approx}]
It follows from the property \eqref{eqn:property_online_regression_v1_loss_general_func_approx} that
\[
    \begin{split}
        \bbE_{\calI^*} \left[ \sum_{i=1}^{n} \overline{l}_i \left\{ \hat{\mu}_i \left( \bfO_{i-1} \right) \right\} \right] = \ & \sum_{i=1}^{n} \bbE_{\calI^*} \left[ \bbE_{\calI^*} \left[ \left. \overline{l}_i \left\{ \hat{\mu}_i \left( \bfO_{i-1} \right) \right\} \right| \left( \calF_{i-1}, X_i, A_i \right) \right] \right] \\
        = \ & \sum_{i=1}^{n} \bbE_{\calI^*} \left[ \sigma^2 \left( X_i, A_i \right) + \left\{ \hat{\mu}_i \left( \bfO_{i-1} \right) \left( X_i, A_i \right) - \mu^* \left( X_i, A_i \right) \right\}^2 \right] \\
        = \ & \sum_{i=1}^{n} \bbE_{\calI^*} \left[ \sigma^2 \left( X_i, A_i \right) \right] + \sum_{i=1}^{n} \bbE_{\calI^*} \left[  \left\{ \hat{\mu}_i \left( \bfO_{i-1} \right) \left( X_i, A_i \right) - \mu^* \left( X_i, A_i \right) \right\}^2 \right],
    \end{split}
\]
which leads to the following expression of the estimation error term \eqref{eqn:treatment_effect_estimation_error_general_func_approx}:
\begin{equation}
    \label{eqn:proof_thm:online_regression_oracle_ineq_general_func_approx_v1}
    \begin{split}
        &\frac{1}{n} \sum_{i=1}^{n} \bbE_{\calI^*} \left[ \left\{ \hat{\mu}_{i} \left( \bfO_{i-1} \right) \left( X_i, A_i \right) - \mu^* \left( X_i, A_i \right) \right\}^2 \right] \\
        = \ & \frac{1}{n} \bbE_{\calI^*} \left[ \sum_{i=1}^{n} \overline{l}_i \left\{ \hat{\mu}_i \left( \bfO_{i-1} \right) \right\} \right] - \frac{1}{n} \sum_{i=1}^{n} \bbE_{\calI^*} \left[ \sigma^2 \left( X_i, A_i \right) \right] \\
        = \ & \frac{1}{n} \bbE_{\calI^*} \left[ \overline{\textsf{Regret}} \left( n; \overline{\calA} \right) \right] + \frac{1}{n} \bbE_{\calI^*} \left[ \inf \left\{ \sum_{i=1}^{n} \overline{l}_i (\mu): \mu \in \calF \right\} \right] - \frac{1}{n} \sum_{i=1}^{n} \bbE_{\calI^*} \left[ \sigma^2 \left( X_i, A_i \right) \right].
    \end{split}
\end{equation}
Here, one may observe that
\begin{equation}
    \label{eqn:proof_thm:online_regression_oracle_ineq_general_func_approx_v2}
    \begin{split}
        &\frac{1}{n} \bbE_{\calI^*} \left[ \inf \left\{ \sum_{i=1}^{n} \overline{l}_i (\mu): \mu \in \calF \right\} \right] \\
        \leq \ & \inf \left\{ \frac{1}{n} \bbE_{\calI^*} \left[ \sum_{i=1}^{n} \overline{l}_i (\mu) \right]: \mu \in \calF \right\} \\
        = \ & \inf \left\{ \frac{1}{n} \sum_{i=1}^{n} \bbE_{\calI^*} \left[ \bbE_{\calI^*} \left[ \left. \overline{l}_i (\mu) \right| \left( \calF_{i-1}, X_i, A_i \right) \right] \right]: \mu \in \calF \right\} \\
        \stackrel{\textnormal{(a)}}{=} \ & \inf \left\{ \frac{1}{n} \sum_{i=1}^{n} \bbE_{\calI^*} \left[ \sigma^2 \left( X_i, A_i \right) + \left\{ \mu \left( X_i, A_i \right) - \mu^* \left( X_i, A_i \right) \right\}^2 \right]: \mu \in \calF \right\} \\
        = \ & \frac{1}{n} \sum_{i=1}^{n} \bbE_{\calI^*} \left[ \sigma^2 \left( X_i, A_i \right) \right] + \inf \left\{ \frac{1}{n} \sum_{i=1}^{n} \bbE_{\calI^*} \left[ \left\{ \mu \left( X_i, A_i \right) - \mu^* \left( X_i, A_i \right) \right\}^2 \right] : \mu \in \calF \right\},  
    \end{split}
\end{equation}
where the step (a) holds by the fact \eqref{eqn:property_online_regression_v1_loss_general_func_approx}. Putting two pieces \eqref{eqn:proof_thm:online_regression_oracle_ineq_general_func_approx_v1} and \eqref{eqn:proof_thm:online_regression_oracle_ineq_general_func_approx_v2} together yields
\begin{equation}
    \label{eqn:proof_thm:online_regression_oracle_ineq_general_func_approx_v3}
    \begin{split}
         &\frac{1}{n} \sum_{i=1}^{n} \bbE_{\calI^*} \left[ \left\{ \hat{\mu}_{i} \left( \bfO_{i-1} \right) \left( X_i, A_i \right) - \mu^* \left( X_i, A_i \right) \right\}^2 \right] \\
         \leq \ & \frac{1}{n} \bbE_{\calI^*} \left[ \overline{\textnormal{Regret}} \left( n; \overline{\calA} \right) \right] + \inf \left\{ \frac{1}{n} \sum_{i=1}^{n} \bbE_{\calI^*} \left[ \left\{ \mu \left( X_i, A_i \right) - \mu^* \left( X_i, A_i \right) \right\}^2 \right] : \mu \in \calF \right\}.
    \end{split}
\end{equation}
Hence, the desired result \eqref{eqn:thm:online_regression_oracle_ineq_general_func_approx_v1} on the \textsf{MSE} for the \textsf{AIPW} estimator \eqref{eqn:alg:aipw_estimator_v1} is a straightforward consequence of the inequality \eqref{eqn:proof_thm:online_regression_oracle_ineq_general_func_approx_v3} by plugging it into the bound \eqref{eqn:general_function_approximation_v1}.

\end{proof}

\indent Here, we remark that aside from the optimal variance $v_{*}^2$, the bound \eqref{eqn:thm:online_regression_oracle_ineq_general_func_approx_v1} shows two additional terms: (\romannumeral 1) the expected regret relative to the number of rounds $n$, where the expectation is taken over $\bfO_n \sim \bbP_{\calI^*}^n (\cdot)$; and (\romannumeral 2) the approximation error term whose form is slightly different from the one $\inf \left\{ \left\| \mu - \mu^* \right\|_{(n)}^{2}: \mu \in \calF \right\}$ appeared in the \textsf{MSE} bound \eqref{eqn:thm:online_regression_oracle_ineq_v1_v1} of Theorem \ref{thm:online_regression_oracle_ineq_v1}.

\paragraph{Non-asymptotic theory of online non-parametric regression}
Before delving into the investigation of the modified regret \eqref{eqn:defi_regret_v2}, we briefly recap the main results in \cite{rakhlin2014online} that establishes a theoretical framework of online non-parametric regression. In contrast to most of the existing works on online regression, the authors do NOT start from an algorithm, but instead directly work with the minimax regret in \cite{rakhlin2014online}. We will be able to extract a (not necessarily efficient) algorithm after taking a closer inspection on the minimax regret. We use $\left\llangle \cdots \right\rrangle_{i=1}^{n}$ to denote an interleaved application of the operators inside repeated over $n$ rounds. With this notation in hand, the minimax regret of the online non-parametric regression problem for estimation of the treatment effect can be written as
\begin{equation}
    \label{eqn:minimax_regret}
    \begin{split}
        \calV_n (\calF) := \left\llangle \sup_{\left( x_i, a_i \right) \in \bbX \times \bbA} \inf_{\hat{y}_i \in \left[ -L, L \right]} \sup_{y_i \in \left[ -L, L \right]} \right\rrangle_{i=1}^{n} \left[ \sum_{i=1}^{n} \left( \hat{y}_i - y_i \right)^2 - \inf_{\mu \in \calF} \sum_{i=1}^{n} \left\{ \mu \left( x_i, a_i \right) - y_i \right\}^2 \right],
    \end{split}
\end{equation}
where $\calF \subseteq \left( \bbX \times \bbA \to \left[ -L, L \right] \right)$ is a pre-specified function class. One of the key tools in the study of estimators based on i.i.d. data is the \emph{symmetrization technique} \cite{gine1984some, wainwright2019high}. Under the i.i.d. scenario, one can investigate the supremum of an empirical process conditionally on the data by introducing Rademacher random variables, which is NOT directly applicable given the adaptive nature of our main problem. Under the online prediction scenario, such a symmetrization technique becomes more subtle and it requires the notion of a binary tree, the smallest entity which captures the sequential nature of the problem in some sense. Towards achieving our goal in our problem, let us state some definitions.

\begin{defi}
\label{defi:tree}
\normalfont{
An \emph{$\bbS$-valued tree of depth n} (here, $\bbS$ is any measurable state space) is defined as a rooted complete binary tree with nodes labeled by elements of the state space $\bbS$: the sequence $\bfs = \left( \bfs_1, \bfs_2, \cdots, \bfs_n \right)$ of labeling functions $\bfs_i (\cdot): \left\{ \pm 1 \right\}^{i-1} \to \bbS$ which provides the labels of each node. Here, $\bfs_1 \in \bbS$ denotes the label for the \emph{root of the tree}, while $\bfs_i$ for $i \in \left\{ 2, 3, \cdots, n \right\}$ denotes the label of the node obtained by following the path of length $i-1$ from the root, with $+1$ indicating \emph{right} and $-1$ indicating \emph{left}. A \emph{path of length $n$} is given by the sequence $\bfepsilon_{1:n} = \left( \epsilon_1, \cdots, \epsilon_n \right) \in \left\{ \pm 1 \right\}^n$. Given any measurable function $\phi (\cdot): \bbS \to \bbR$, $\phi (\bfs)$ is an $\bbR$-valued tree of depth $n$ with labeling functions $\left( \phi \circ \bfs_{i} \right) (\cdot): \left\{ \pm 1 \right\}^{i-1} \to \bbR$ for level $i \in [n]$ (or, in words, the evaluation of $\phi (\cdot): \bbS \to \bbR$, $\phi (\bfs)$ on $\bfs$). Lastly, let $\textsf{Tree} \left( \bbS, n \right)$ denote the set of all $\bbS$-valued trees of depth $n$.
}
\end{defi}

\noindent Here, one may think of the sequence of functions $\left\{ \bfs_i (\cdot) : i \in [n] \right\}$ defined on the underlying sample space as a predictable stochastic process with respect to the dyadic filtration $\left\{ \sigma \left( \bfepsilon_{1:i} \right) : i \in [n] \right\}$. Next, let us define the notion of a \emph{sequential $\beta$-cover} quantifies one of the key complexity measures of a function class $\calG \subseteq \left( \bbS \to \bbR \right)$ evaluated on the predictable process: the \emph{sequential covering number}.

\begin{defi} [Sequential covering numbers \cite{rakhlin2015sequential}]
\label{defi:sequential_covering_num}
\normalfont{\
\begin{enumerate} [label = (\roman*)]
    \item We define a random pseudo-metric between two $\bbR$-valued trees $\bfu = \left( \bfu_i : i \in [n] \right)$ and $\bfv = \left( \bfv_i : i \in [n] \right)$ of depth $n$ defined as follows: for any $\left( p, \bfepsilon_{1:n} \right) \in \left[ 1, +\infty \right] \times \left\{ \pm 1 \right\}^n$,
    \begin{equation}
        \label{eqn:defi:sequential_covering_num_v1}
        \begin{split}
            d_{\bfepsilon_{1:n}}^p \left( \bfu, \bfv \right) :=
            \begin{cases}
                \left\{ \frac{1}{n} \sum_{i=1}^{n} \left| \bfu_i \left( \bfepsilon_{1:i-1} \right) - \bfv_i \left( \bfepsilon_{1:i-1} \right) \right|^p \right\}^{\frac{1}{p}} & \textnormal{if } 1 \leq p < +\infty; \\
                \max \left\{ \left| \bfu_i \left( \bfepsilon_{1:i-1} \right) - \bfv_i \left( \bfepsilon_{1:i-1} \right) \right|: i \in [n] \right\} & \textnormal{if } p = +\infty.
            \end{cases}
        \end{split}
    \end{equation}
    \item A set $V \subseteq \textsf{Tree} \left( \bbR, n \right)$ is called a \emph{sequential $\beta$-cover with respect to the $l_p$-norm of $\calG \subseteq \left( \bbS \to \bbR \right)$ on a given $\bbS$-valued tree $\bfs$ of depth $n$}, where $p \in \left[ 1, +\infty \right]$, if
    \begin{equation}
        \label{eqn:defi:sequential_covering_num_v2}
        \begin{split}
            \sup \left\{ \inf \left\{ d_{\bfepsilon_{1:n}}^p \left( \bfu, \bfv \right): \bfv \in V \right\} : \left( \bfu, \bfepsilon_{1:n} \right) \in \calG (\bfs) \times \left\{ \pm 1 \right\}^n \right\} \leq \beta,
        \end{split}
    \end{equation}
    where $\calG (\bfs) := \left\{ g (\bfs): g \in \calG \right\} \subseteq \textsf{Tree} \left( \bbR, n \right)$;
    \item The \emph{sequential $\beta$-covering number with respect to the $l_p$-norm of a function class $\calG \subseteq \left( \bbS \to \bbR \right)$ on an $\bbS$-valued tree $\bfs$ of depth $n$}, where $p \in \left[ 1, +\infty \right]$, is defined by
    \begin{equation*}
        \begin{split}
            \calN_p \left( \beta, \calG, \bfs \right) := \min \left\{ \left| V \right|: V \subseteq \textsf{Tree} \left( \bbR, n \right) \textnormal{ is a sequential $\beta$-cover w.r.t. the $l_p$-norm of $\calG$ on $\bfs$} \right\}.
        \end{split}
    \end{equation*}
    Here, we further define $\calN_p \left( \beta, \calG, n \right) := \sup \left\{ \calN_p \left( \beta, \calG, \bfs \right): \bfs \in \textsf{Tree} \left( \bbS, n \right) \right\}$ to be the \emph{maximal sequential $\beta$-covering number with respect to the $l_p$-norm of $\calG$ over $\bbS$-valued trees of depth $n$}. Now, we will refer to $\log \calN_p \left( \beta, \calG, n \right)$ as the \emph{sequential $\beta$-metric entropy of $\calG$ with respect to the $l_p$-norm}.
\end{enumerate}
}
\end{defi}

In particular, we investigate the behavior of the minimax regret $\calV_n (\calF)$ for the case where the sequential metric entropy of the function class $\calF \subseteq \left( \bbX \times \bbA \to \left[ -L, L \right] \right)$ with respect to the $l_2$-norm grows polynomially as the scale $\beta$ decreases:
\begin{equation}
    \label{eqn:cond_sequential_entropy_v1}
    \begin{split}
        \log \calN_2 \left( \beta, \calF, n \right) \sim \beta^{-p} \quad \textnormal{for } p \in \left( 0, +\infty \right).
    \end{split}
\end{equation}
Let us also consider the \emph{parametric ``$p = 0$'' case} when the sequential covering number of $\calF$ with respect to $l_2$-norm itself behaves as:
\begin{equation}
    \label{eqn:cond_sequential_entropy_v2}
    \begin{split}
        \calN_2 \left( \beta, \calF, n \right) \sim \beta^{-d}.
    \end{split}
\end{equation}
For instance, the function class $\calF := \left\{ f_{\bftheta} (\cdot): \bbR^d \to \bbR: \bftheta \in \Theta \right\}$ for the linear regression problem in a bounded measurable subset $\Theta \subseteq \bbR^d$, where the function $f_{\bftheta} (\cdot): \bbR^d \to \bbR$ is given by $f_{\bftheta} (\bfx) := \bftheta^{\top} \bfx$ for $\bftheta \in \bbR^d$, satisfies the condition \eqref{eqn:cond_sequential_entropy_v2}. By employing the main results (in particular, \emph{Theorem 2}) in \cite{rakhlin2014online}, one can establish the following conclusion:

\begin{thm} [The rates of convergence of the minimax regret]
\label{thm:conv_rates_minimax_regret_v1}
Given any function class $\calF \subseteq \left( \bbX \times \bbA \to \left[ -L, L \right] \right)$ with sequential metric entropy growth $\log \calN_2 \left( \beta, \calF, n \right) \leq \beta^{-p}$ for $p \in \left( 0, +\infty \right)$, it holds that
\begin{enumerate} [label = (\roman*)]
    \item for $p \in \left( 2, +\infty \right)$, the minimax regret \eqref{eqn:minimax_regret} is bounded as
    \begin{equation}
        \label{eqn:thm:conv_rates_minimax_regret_v1_v1}
        \begin{split}
            \calV_n (\calF) \leq \left( 4 + \frac{24}{p-2} \right) L n^{1 - \frac{1}{p}}.
        \end{split}
    \end{equation}
    \item for $p \in \left( 0, 2 \right)$, the minimax regret \eqref{eqn:minimax_regret} is bounded as
    \begin{equation}
        \label{eqn:thm:conv_rates_minimax_regret_v1_v2}
        \begin{split}
            \calV_n (\calF) \leq \left( 32 L^2 + 4L + \frac{24L}{2-p} \right) n^{1 - \frac{2}{p+2}}.
        \end{split}
    \end{equation}
    \item for $p = 2$, the minimax regret \eqref{eqn:minimax_regret} is bounded as
    \begin{equation}
        \label{eqn:thm:conv_rates_minimax_regret_v1_v3}
        \begin{split}
            \calV_n (\calF) \leq \left( 32 L^2 + 4L + 3 \right) \sqrt{n} \log n.
        \end{split}
    \end{equation}
    \item for the parametric case \eqref{eqn:cond_sequential_entropy_v2}, the minimax regret \eqref{eqn:minimax_regret} is bounded as
    \begin{equation}
        \label{eqn:thm:conv_rates_minimax_regret_v1_v4}
        \begin{split}
            \calV_n (\calF) \leq \left( 16 L^2 + 4L + 12 \right) d \log n.
        \end{split}
    \end{equation}
    \item if the function class $\calF \subseteq \left( \bbX \times \bbA \to \left[ -L, L \right] \right)$ is a finite set, the minimax regret \eqref{eqn:minimax_regret} is bounded as
    \begin{equation}
        \label{eqn:thm:conv_rates_minimax_regret_v1_v5}
        \begin{split}
            \calV_n (\calF) \leq 32 L^2 \log \left| \calF \right|.
        \end{split}
    \end{equation}
\end{enumerate}
\end{thm}

\noindent It is shown in \cite{rakhlin2014online} that the upper bounds (\romannumeral 1)--(\romannumeral 4) on the minimax regret \eqref{eqn:minimax_regret} in Theorem \ref{thm:conv_rates_minimax_regret_v1} are \emph{tight up to logarithmic factors}. See \emph{Theorem 3} therein for further details.
\medskip

Although Theorem \ref{thm:conv_rates_minimax_regret_v1} characterizes the rates of convergence of the minimax regret \eqref{eqn:minimax_regret} under various scenarios \emph{statistically}, its proof is \emph{non-constructive} in the sense that the regret bounds therein are established without explicitly constructing an algorithm. To provide a general algorithmic framework for the problem of online non-parametric regression, we follow the abstract \emph{relaxation recipe} proposed in \cite{rakhlin2012relax}. It was shown therein that if one can find a sequence of mappings from the observed data to real numbers, often called a \emph{relaxation}, satisfying some desirable conditions, then one can construct estimators based on such relaxations. To be specific, we search for a relaxation $\textsf{Rel}_n \left( \cdot, \cdot \right): \biguplus_{k=0}^{n} \left\{ \left( \bbX \times \bbA \right)^k \times \left[ -L, L \right]^k \right\} \to \bbR$ satisfying the following two conditions:

\begin{assumption} [Initial condition]
\label{assumption:relaxation_initial_cond}
\normalfont{
The relaxation $\textsf{Rel}_n \left( \cdot, \cdot \right): \biguplus_{k=0}^{n} \left\{ \left( \bbX \times \bbA \right)^k \times \left[ -L, L \right]^k \right\} \to \bbR$ satisfies
\begin{equation}
    \label{eqn:assumption:relaxation_initial_cond}
    \begin{split}
        \textsf{Rel}_n \left( \left( \bfx, \bfa \right)_{1:n}, \bfy_{1:n} \right) \geq - \inf \left\{ \sum_{k=1}^{n} \left\{ y_i - \mu \left( x_i, a_i \right) \right\}^2: \mu (\cdot, \cdot) \in \calF \right\},
    \end{split}
\end{equation}
where $\left( \bfx, \bfa \right)_{1:k} := \left( \left( x_i, a_i \right): i \in [k] \right) \in \left( \bbX \times \bbA \right)^k$ and $\bfy_{1:k} := \left( y_i : i \in [k] \right) \in \left[ -L, L \right]^k$ for every $k \in [n]$.
}
\end{assumption}

\begin{assumption} [Recursive admissibility condition]
\label{assumption:relaxation_recursive_admissibility}
\normalfont{
The relaxation $\textsf{Rel}_n \left( \cdot, \cdot \right): \biguplus_{k=0}^{n} \left\{ \left( \bbX \times \bbA \right)^k \times \left[ -L, L \right]^k \right\} \to \bbR$ satisfies
\begin{equation}
    \label{eqn:assumption:relaxation_recursive_admissibility}
    \begin{split}
        \inf_{\hat{y}_k \in \left[ -L, L \right]} \sup_{y_k \in \left[ -L, L \right]} \left\{ \left( \hat{y}_k - y_k \right)^2 + \textsf{Rel}_n \left( \left( \bfx, \bfa \right)_{1:k}, \bfy_{1:k} \right) \right\} \leq \textsf{Rel}_n \left( \left( \bfx, \bfa \right)_{1:k-1}, \bfy_{1:k-1} \right),
    \end{split}
\end{equation}
for any $k \in [n]$ and any $x_k \in \bbX$.
}
\end{assumption}

\noindent A relaxation $\textsf{Rel}_n \left( \cdot, \cdot \right): \biguplus_{k=0}^{n} \left\{ \left( \bbX \times \bbA \right)^k \times \left[ -L, L \right]^k \right\} \to \bbR$ that satisfies Assumptions \ref{assumption:relaxation_initial_cond} and \ref{assumption:relaxation_recursive_admissibility} is said to be \emph{admissible}. With an admissible relaxation $\textsf{Rel}_n \left( \cdot, \cdot \right)$ in hand, one can design an algorithm for the online non-parametric regression problem with the following associated regret bound (see Algorithm \ref{alg:generic_forecaster} for the detailed description):
\begin{equation}
    \label{eqn:regret_bound_v1_alg:generic_forecaster}
    \begin{split}
        \overline{\textsf{Regret}} \left( n, \calF; \textnormal{Alg. \ref{alg:generic_forecaster}} \right) = \ & \sum_{i=1}^{n} \left\{ Y_i - \hat{\mu}_i \left( \bfO_{i-1} \right) \left( X_i, A_i \right) \right\}^2 - \inf \left\{ \sum_{i=1}^{n} \left\{ Y_i - \mu \left( X_i, A_i \right) \right\}^2 : \mu \in \calF \right\} \\
        \leq \ & \textsf{Rel}_n \left( \varnothing, \varnothing \right).
    \end{split}
\end{equation}

\begin{algorithm}[h!]
\caption{A generic forecaster based on the relaxation recipe proposed in \cite{rakhlin2012relax}}
\label{alg:generic_forecaster}
\begin{algorithmic}[1]
    \Require{a relaxation $\textsf{Rel}_n \left( \cdot, \cdot \right): \biguplus_{k=0}^{n} \left\{ \left( \bbX \times \bbA \right)^k \times \left[ -L, L \right]^k \right\} \to \bbR$.}
    \State We first choose $\hat{\mu}_1 (\varnothing) (\cdot, \cdot) \in \left( \bbX \times \bbA \to \bbR \right)$ as
    \begin{equation}
        \label{eqn:alg:generic_forecaster_v1}
        \begin{split}
            \hat{\mu}_1 (\varnothing) (x, a) \in \argmin \left\{ \sup_{y_1 \in \left[ -L, L \right]} \left\{ \left( \hat{y} - y_1 \right)^2 + \textnormal{\textsf{Rel}}_n \left( (x, a), y_1 \right) \right\}: \hat{y} \in \left[ -L, L \right] \right\}.
        \end{split}
    \end{equation}
    \For{$i = 2, 3, \cdots, n$,}
        \State Observe a triple $\left( X_i, A_i, Y_i \right) \in \bbO$; 
        \State We compute $\hat{\mu}_i \left( \bfO_{i-1} \right) \in \left( \bbX \times \bbA \to \bbR \right)$ according to the following rule:
        \begin{equation}
            \label{eqn:alg:generic_forecaster_v2}
            \begin{split}
                &\hat{\mu}_i \left( \bfO_{i-1} \right) (x, a) \\
                \in \ & \arg \min \left\{ \sup_{y_i \in \left[ -L, L \right]} \left\{ \left( \hat{y} - y_i \right)^2 + \textsf{Rel}_n \left( \left( \left( \bfX, \bfA \right)_{1:i-1}, (x, a) \right), \left( \bfY_{1:i-1}, y_i \right) \right) \right\}: \hat{y} \in \left[ -L, L \right] \right\}.
            \end{split}
        \end{equation}
    \EndFor
    \State \Return the sequence of estimates $\left\{ \hat{\mu}_i \left( \bfO_{i-1} \right) \in \left( \bbX \times \bbA \to \bbR \right): i \in [n] \right\}$ of the treatment effect.
\end{algorithmic}
\end{algorithm}

\indent One can further see that if the function $y_i \in \left[ -L, L \right] \mapsto \left( \hat{y} - y_i \right)^2 + \textsf{Rel}_n \left( \left( \left( \bfx, \bfa \right)_{1:i} \right), \left( \bfy_{1:i-1}, y_i \right) \right)$ is convex for every $\left( \hat{y}, \bfx_{1:n}, \bfa_{1:n}, \bfy_{1:i-1} \right) \in \left[ -L, L \right] \times \bbX^n \times \bbA^n \times \left[ -L, L \right]^{i-1}$ and $i \in [n]$, then the prediction rules \eqref{eqn:alg:generic_forecaster_v1} and \eqref{eqn:alg:generic_forecaster_v2} becomes much simpler, because the supremum over $y_i \in \left[ -L, L \right]$ is attained either $L$ or $-L$. The prediction rules then can be written as
\begin{equation}
    \label{eqn:alg:generic_forecaster_auxiliary_form_v1}
    \begin{split}
        \hat{\mu}_1 (\varnothing) (x, a) \in \arg \min \left\{ \max \left\{ \left( \hat{y} - L \right)^2 + \textsf{Rel}_n \left( (x, a), L \right), \left( \hat{y} + L \right)^2 + \textsf{Rel}_n \left( (x, a), - L \right) \right\}: \hat{y} \in \left[ -L, L \right] \right\},
    \end{split}
\end{equation}
and for $i \in \left\{ 2, 3, \cdots, n \right\}$,
\begin{equation}
    \label{eqn:alg:generic_forecaster_auxiliary_form_v2}
    \begin{split}
        \hat{\mu}_i \left( \bfO_{i-1} \right) (x, a) \in \ & \arg \min \left\{ \max \left\{ \left( \hat{y} - L \right)^2 + \textsf{Rel}_n \left( \left( \left( \bfX, \bfA \right)_{1:i-1}, (x, a) \right), \left( \bfY_{1:i-1}, L \right) \right), \right. \right. \\
        &\left. \left. \left( \hat{y} + L \right)^2 + \textsf{Rel}_n \left( \left( \left( \bfX, \bfA \right)_{1:i-1}, (x, a) \right), \left( \bfY_{1:i-1}, - L \right) \right) \right\}: \hat{y} \in \left[ -L, L \right] \right\}.
    \end{split}
\end{equation}
One can easily observe that the prediction rules \eqref{eqn:alg:generic_forecaster_auxiliary_form_v1} and \eqref{eqn:alg:generic_forecaster_auxiliary_form_v2} can be further simplified as
\begin{equation}
    \label{eqn:alg:generic_forecaster_simplified_form_v1}
    \begin{split}
        \hat{\mu}_1 (\varnothing) (x, a) = \chi_{\left[ -L, L \right]} \left\{ \frac{\textsf{Rel}_n \left( (x, a), L \right) - \textsf{Rel}_n \left( (x, a), -L \right)}{4L} \right\},
    \end{split}
\end{equation}
and for $i \in \left\{ 2, 3, \cdots, n \right\}$,
\begin{equation}
    \label{eqn:alg:generic_forecaster_simplified_form_v2}
    \begin{split}
        &\hat{\mu}_i \left( \bfO_{i-1} \right) (x, a) \\
        = \ & \chi_{\left[ -L, L \right]} \left\{ \frac{\textsf{Rel}_n \left( \left( \left( \bfX, \bfA \right)_{1:i-1}, (x, a) \right), \left( \bfY_{1:i-1}, L \right) \right) - \textsf{Rel}_n \left( \left( \left( \bfX, \bfA \right)_{1:i-1}, (x, a) \right), \left( \bfY_{1:i-1}, -L \right) \right)}{4L} \right\},
    \end{split}
\end{equation}
where $\chi_{\left[ -L, L \right]} (\cdot): \bbR \to \left[ -L, L \right]$ defines a clip function onto the interval $\left[ -L, L \right]$, i.e.,
\[
    \chi_{\left[ -L, L \right]} (x) :=
    \begin{cases}
        L & \textnormal{if } x > L; \\
        x & \textnormal{if } -L \leq x \leq L; \\
        -L & \textnormal{otherwise.}
    \end{cases}
\]

\indent By directly using \emph{Lemma 16} in \cite{rakhlin2014online}, one can obtain the following significant result:

\begin{thm}
\label{thm:admissible_relaxation_v1}
The relaxation $\calR_n \left( \cdot, \cdot \right): \biguplus_{k=0}^{n} \left\{ \left( \bbX \times \bbA \right)^k \times \left[ -L, L \right]^k \right\} \to \bbR$ defined as
\begin{equation}
    \label{eqn:thm:admissible_relaxation_v1_v1}
    \begin{split}
        \calR_n \left( \left( \bfx, \bfa \right)_{1:k}, \bfy_{1:k} \right) := \ & \sup_{\left( \bfz, \bfm \right)} \bbE_{\bfepsilon_{1:n} \sim \textnormal{\textsf{Unif}} \left( \left\{ \pm 1 \right\}^n \right)} \left[ \sup \left\{ \sum_{j = k+1}^{n} \left[ 4 L \epsilon_j \left\{ \mu \left( \bfz_j \left( \bfepsilon_{1:j-1} \right) \right) - \bfm_j \left( \bfepsilon_{1:j-1} \right) \right\} \right. \right. \right. \\
        &\left. \left. \left. - \left\{ \mu \left( \bfz_j \left( \bfepsilon_{1:j-1} \right) \right) - \bfm_j \left( \bfepsilon_{1:j-1} \right) \right\}^2 \right] - \sum_{j=1}^{k} \left\{ \mu \left( x_j, a_j \right) - y_j \right\}^2: \mu \in \calF \right\} \right],
    \end{split}
\end{equation}
where the pair $\left( \bfz, \bfm \right)$ ranges over $\textnormal{\textsf{Tree}} \left( \bbX \times \bbA, n \right) \times \textnormal{\textsf{Tree}} \left( \bbR, n \right)$, is admissible. Using the regret bound \eqref{eqn:regret_bound_v1_alg:generic_forecaster}, one can conclude that Algorithm \ref{alg:generic_forecaster} using the admissible relaxation $\calR_n \left( \cdot, \cdot \right): \biguplus_{k=0}^{n} \left\{ \left( \bbX \times \bbA \right)^k \times \left[ -L, L \right]^k \right\} \to \bbR$ as an input enjoys the regret bound of an offset Rademacher complexity:
\begin{equation}
    \label{eqn:thm:admissible_relaxation_v1_v2}
    \begin{split}
        \overline{\textnormal{\textsf{Regret}}} \left( n, \calF; \textnormal{Alg. \ref{alg:generic_forecaster}} \right) \leq \ & \calR_n \left( \varnothing, \varnothing \right) \\
        = \ & \sup_{\left( \bfz, \bfm \right)} \bbE_{\bfepsilon_{1:n} \sim \textnormal{\textsf{Unif}} \left( \left\{ \pm 1 \right\}^n \right)} \left[ \sup \left\{ \sum_{j = 1}^{n} \left[ 4 L \epsilon_j \left\{ \mu \left( \bfz_j \left( \bfepsilon_{1:j-1} \right) \right) - \bfm_j \left( \bfepsilon_{1:j-1} \right) \right\} \right. \right. \right. \\
        &\left. \left. \left. - \left\{ \mu \left( \bfz_j \left( \bfepsilon_{1:j-1} \right) \right) - \bfm_j \left( \bfepsilon_{1:j-1} \right) \right\}^2 \right]: \mu \in \calF \right\} \right].
    \end{split}
\end{equation}
\end{thm}

\noindent Since the upper bounds on the minimax regret \eqref{eqn:minimax_regret} provided in Theorem \ref{thm:conv_rates_minimax_regret_v1} are established by further upper bounding the offset Rademacher complexity $\calR_n \left( \varnothing, \varnothing \right)$, one can end up with the following corollary:

\begin{cor}
\label{cor:admissible_relaxation_v1}
We consider a function class $\calF \subseteq \left( \bbX \times \bbA \to \left[ -L, L \right] \right)$ with sequential metric entropy growth $\log \calN_2 \left( \beta, \calF, n \right) \leq \beta^{-p}$ for $p \in \left( 0, +\infty \right)$. Then, Algorithm \ref{alg:generic_forecaster} utilizing the admissible relaxation $\calR_n \left( \cdot, \cdot \right)$ defined by \eqref{eqn:thm:admissible_relaxation_v1_v1} as an input enjoys the following regret bounds:
\begin{enumerate} [label = (\roman*)]
    \item for $p \in \left( 2, +\infty \right)$, it holds that
    \begin{equation}
        \label{eqn:cor:admissible_relaxation_v1_v1}
        \begin{split}
            \overline{\textnormal{\textsf{Regret}}} \left( n, \calF; \textnormal{Algorithm \ref{alg:generic_forecaster}} \right) \leq \left( 4 + \frac{24}{p-2} \right) L n^{1 - \frac{1}{p}}.
        \end{split}
    \end{equation}
    \item for $p \in \left( 0, 2 \right)$, it holds that
    \begin{equation}
        \label{eqn:cor:admissible_relaxation_v1_v2}
        \begin{split}
            \overline{\textnormal{\textsf{Regret}}} \left( n, \calF; \textnormal{Algorithm \ref{alg:generic_forecaster}} \right) \leq \left( 32 L^2 + 4L + \frac{24L}{2-p} \right) n^{1 - \frac{2}{p+2}}.
        \end{split}
    \end{equation}
    \item for $p = 2$, it holds that
    \begin{equation}
        \label{eqn:cor:admissible_relaxation_v1_v3}
        \begin{split}
            \overline{\textnormal{\textsf{Regret}}} \left( n, \calF; \textnormal{Algorithm \ref{alg:generic_forecaster}} \right) \leq \left( 32 L^2 + 4L + 3 \right) \sqrt{n} \log n.
        \end{split}
    \end{equation}
    \item for the parametric case \eqref{eqn:cond_sequential_entropy_v2}, it holds that
    \begin{equation}
        \label{eqn:cor:admissible_relaxation_v1_v4}
        \begin{split}
            \overline{\textnormal{\textsf{Regret}}} \left( n, \calF; \textnormal{Algorithm \ref{alg:generic_forecaster}} \right) \leq \left( 16 L^2 + 4L + 12 \right) d \log n.
        \end{split}
    \end{equation}
    \item if the function class $\calF \subseteq \left( \bbX \times \bbA \to \left[ -L, L \right] \right)$ is a finite set, it holds that
    \begin{equation}
        \label{eqn:cor:admissible_relaxation_v1_v5}
        \begin{split}
            \overline{\textnormal{\textsf{Regret}}} \left( n, \calF; \textnormal{Algorithm \ref{alg:generic_forecaster}} \right) \leq 32 L^2 \log \left| \calF \right|.
        \end{split}
    \end{equation}
\end{enumerate}
\end{cor}

\indent Although Corollary \ref{cor:admissible_relaxation_v1} gives no-regret learning guarantees of Algorithm \ref{alg:generic_forecaster} with the admissible relaxation $\calR_n \left( \cdot, \cdot \right)$ defined by \eqref{eqn:thm:admissible_relaxation_v1_v1} for a number of function classes $\calF \subseteq \left( \bbX \times \bbA \to \left[ -L, L \right] \right)$, it is still NOT a practical algorithm because the relaxation $\calR_n \left( \cdot, \cdot \right)$ defined as \eqref{eqn:thm:admissible_relaxation_v1_v1} is not directly computable in general. To address this problem, \cite{rakhlin2014online} provided a generic schema for developing implementable online non-parametric regression algorithms. The schema can be described as follows:
\begin{enumerate} [label = (\alph*)]
    \item Find a \emph{computable relaxation} $\textsf{Rel}_n \left( \cdot, \cdot \right): \biguplus_{k=0}^{n} \left\{ \left( \bbX \times \bbA \right)^k \times \left[ -L, L \right]^k \right\} \to \bbR$ such that
    \[
        \begin{split}
            \calR_n \left( \left( \bfx, \bfa \right)_{1:k}, \bfy_{1:k} \right) \leq \textsf{Rel}_n \left( \left( \bfx, \bfa \right)_{1:k}, \bfy_{1:k} \right)
        \end{split}
    \]
    for every $\left( k, \bfx_{1:n}, \bfa_{1:n}, \bfy_{1:n} \right) \in \left\{ 0, 1, \cdots, n \right\} \times \bbX^n \times \bbA^n \times \left[ -L, L \right]^n$, and the function $y_k \in \left[ -L, L \right] \mapsto \left( \hat{y} - y_k \right)^2 + \textsf{Rel}_n \left( \left( \left( \bfx, \bfa \right)_{1:k} \right), \left( \bfy_{1:k-1}, y_k \right) \right) \in \bbR$ is convex for every $\left( \hat{y}, \bfx_{1:n}, \bfa_{1:n}, \bfy_{1:k-1} \right) \in \left[ -L, L \right] \times \bbX^n \times \bbA^n \times \left[ -L, L \right]^{k-1}$ and $k \in [n]$;
    \item Next, we check the following condition:
    \[
        \begin{split}
            &\sup_{\left( x_k, a_k, \mu_k \right) \in \bbX \times \bbA \times \Delta \left( \left[ -L, L \right] \right)} \left\{ \bbE_{y_k \sim \mu_k} \left[ \left( \bbE_{y_k \sim \mu_k} \left[ y_k \right] - y_k \right)^2 \right] + \bbE_{y_k \sim \mu_k} \left[ \textsf{Rel}_n \left( \left( \bfx, \bfa \right)_{1:k}, \bfy_{1:k} \right) \right] \right\} \\
            \leq \ & \textsf{Rel}_n \left( \left( \bfx, \bfa \right)_{1:k-1}, \bfy_{1:k-1} \right)
        \end{split}
    \]
    for every $\left( \bfx_{1:k-1}, \bfa_{1:k-1}, \bfy_{1:k-1} \right) \in \bbX^{k-1} \times \bbA_{k-1} \times \left[ -L, L \right]^{k-1}$ and $k \in [n]$;
    \item Implement Algorithm \ref{alg:generic_forecaster} using the relaxation $\textsf{Rel}_n \left( \cdot, \cdot \right): \biguplus_{k=0}^{n} \left\{ \left( \bbX \times \bbA \right)^k \times \left[ -L, L \right]^k \right\} \to \bbR$ as an input.
\end{enumerate}
The authors proved that any computable relaxation $\textsf{Rel}_n \left( \cdot, \cdot \right): \biguplus_{k=0}^{n} \left\{ \left( \bbX \times \bbA \right)^k \times \left[ -L, L \right]^k \right\} \to \bbR$ satisfying conditions in (a) and (b) are admissible; see \emph{Proposition 17} therein. Consequently, any online non-parametric regression algorithm produced by the above generic schema always satisfies the regret bound \eqref{eqn:regret_bound_v1_alg:generic_forecaster}. Moreover, the authors established a practical online non-parametric regression algorithm with no-regret learning guarantees based on the above procedure for the finite function class $\calF \subseteq \left( \bbX \times \bbA \to \left[ -L, L \right] \right)$ and the online linear regression problem.

\section{Proofs for Section \ref{sec:lower_bounds}}
\label{sec:proof_sec:lower_bounds}

\subsection{Proof of Theorem \ref{thm:local_minimax_lower_bound}}
\label{subsec:proof_thm:local_minimax_lower_bound}

Theorem \ref{thm:local_minimax_lower_bound} can be established by taking the following two lemmas collectively:

\begin{lemma}
\label{lemma:local_minimax_lower_bound_v1}
Under Assumption \ref{assumption:mr_v1}, the local minimax risk over the class $\calC_{\delta} \left( \calI^* \right)$ is lower bounded by
\begin{equation}
    \label{eqn:lemma:local_minimax_lower_bound_v1_v1}
    \begin{split}
        \calM_n \left( \calC_{\delta} \left( \calI^* \right) \right) \geq \frac{1}{2304} \left( 1 - \frac{1}{\sqrt{2}} \right) \cdot \frac{1}{n} \textnormal{\textsf{Var}}_{X \sim \Xi^*} \left[ \left\langle g (X, \cdot), \mu^* (X, \cdot) \right\rangle_{\lambda_{\bbA}} \right],
    \end{split}
\end{equation}
provided that $n \geq 16 H_{2 \to 4}^2$.
\end{lemma}

\begin{lemma}
\label{lemma:local_minimax_lower_bound_v2}
Under Assumption \ref{assumption:ln}, the local minimax risk over the class $\calC_{\delta} \left( \calI^* \right)$ is lower bounded by
\begin{equation}
    \label{eqn:lemma:local_minimax_lower_bound_v2_v1}
    \begin{split}
        \calM_n \left( \calC_{\delta} \left( \calI^* \right) \right) \geq \frac{1}{8 K^4} \cdot \frac{\left\| \sigma \right\|_{(n)}^2}{n}.
    \end{split}
\end{equation}
\end{lemma}

\subsection{Proof of Lemma \ref{lemma:local_minimax_lower_bound_v1}}
\label{subsec:proof_lemma:local_minimax_lower_bound_v1}

The proof relies on Le Cam's two-point method by taking the outcome kernel $\Gamma^*: \bbX \times \bbA \to \Delta (\bbY)$ to be fixed, and perturbing the context distribution $\Xi^* (\cdot) \in \Delta (\bbX)$: we first construct a collection of context distributions $\left\{ \Xi_s (\cdot) \in \Delta (\bbX): s \in \left( 0, +\infty \right) \right\}$. Later, we will choose the parameter $s > 0$ small enough so that $\Xi_s (\cdot) \in \calN \left( \Xi^* \right)$ and two distributions $\bbP_{\left( \Xi_s, \Gamma^* \right)}^{n} (\cdot) \in \Delta \left( \bbO^n \right)$ and $\bbP_{\left( \Xi^*, \Gamma^* \right)}^{n} (\cdot) \in \Delta \left( \bbO^n \right)$ are \emph{indistinguishable}, but large enough such that the functional values $\tau \left( \Xi_s, \Gamma^* \right)$ and $\tau \left( \Xi^*, \Gamma^* \right)$ are \emph{well-separated}. Le Cam's two-point lemma (the equation (15.14) in \cite{wainwright2019high}) guarantees that the local minimax risk $\calM_n \left( \calC_{\delta} \left( \calI^* \right) \right)$ is lower bounded as
\begin{equation}
    \label{eqn:proof_thm:local_minimax_lower_bound_v1_v1}
    \begin{split}
        \calM_n \left( \calC_{\delta} \left( \calI^* \right) \right) \geq \ & \frac{1}{4} \left\{ 1 - \textsf{TV} \left( \bbP_{\left( \Xi_s, \Gamma^* \right)}^{n}, \bbP_{\left( \Xi^*, \Gamma^* \right)}^{n} \right) \right\} \left\{ \tau \left( \Xi_s, \Gamma^* \right) - \tau \left( \Xi^*, \Gamma^* \right) \right\}^2,
    \end{split}
\end{equation}
provided that $\Xi_s \in \calN \left( \Xi^* \right)$.
\medskip

\indent First, we upper bound the total variation distance $\textsf{TV} \left( \bbP_{\left( \Xi_s, \Gamma^* \right)}^{n}, \bbP_{\left( \Xi^*, \Gamma^* \right)}^{n} \right)$. Thanks to the Pinsker-Csisz\'{a}r-Kullback inequality, one has
\begin{equation}
    \label{eqn:proof_thm:local_minimax_lower_bound_v1_v2}
    \begin{split}
        \textsf{TV} \left( \bbP_{\left( \Xi_s, \Gamma^* \right)}^{n}, \bbP_{\left( \Xi^*, \Gamma^* \right)}^{n} \right) \leq \sqrt{\frac{1}{2}
        \textsf{KL} \left( \left. \bbP_{\left( \Xi_s, \Gamma^* \right)}^{n} \right\| \bbP_{\left( \Xi^*, \Gamma^* \right)}^{n} \right)}.
    \end{split}
\end{equation}
One can reveal that the density function of the law $\bbP_{\calI}^{n} = \bbP_{\left( \Xi, \Gamma \right)}^{n} \in \Delta \left( \bbO^n \right)$ of the sample trajectory $\bfO_n$ under the problem instance $\calI = \left( \Xi, \Gamma \right) \in \bbI$ with respect to the base measure $\left( \lambda_{\bbX} \otimes \lambda_{\bbA} \otimes \lambda_{\bbA} \right)^{\otimes n}$ is given by
\begin{equation}
    \label{eqn:proof_thm:local_minimax_lower_bound_v1_v3}
    \begin{split}
        p_{\calI}^n \left( \bfo_n \right) = p_{\left( \Xi, \Gamma \right)}^{n} \left( \bfo_n \right) = \prod_{i=1}^{n} \left\{ \xi \left( x_i \right) \pi_{i}^* \left( x_i, \bfo_{i-1}; a_i \right) \gamma \left( \left. y_i \right| x_i, a_i \right) \right\}.
    \end{split}
\end{equation}
Using this fact, the \textsf{KL}-divergence $\textsf{KL} \left( \left. \bbP_{\left( \Xi_s, \Gamma^* \right)}^{n} \right\| \bbP_{\left( \Xi^*, \Gamma^* \right)}^{n} \right)$ can be computed as
\begin{equation}
    \label{eqn:proof_thm:local_minimax_lower_bound_v1_v4}
    \begin{split}
        \textsf{KL} \left( \left. \bbP_{\left( \Xi_s, \Gamma^* \right)}^{n} \right\| \bbP_{\left( \Xi^*, \Gamma^* \right)}^{n} \right) = \ & \bbE_{\left( \Xi_s, \Gamma^* \right)} \left[ \log \frac{p_{\left( \Xi_s, \Gamma^* \right)}^n \left( \bfO_n \right)}{p_{\left( \Xi^*, \Gamma^* \right)}^n \left( \bfO_n \right)} \right] \\
        = \ & \bbE_{\left( \Xi_s, \Gamma^* \right)} \left[ \sum_{i=1}^{n} \log \frac{\xi_s \left( X_i \right) \pi_{i}^* \left( X_i, \bfO_{i-1}; A_i \right) \gamma^* \left( \left. Y_i \right| X_i, A_i \right)}{\xi^* \left( X_i \right) \pi_{i}^* \left( X_i, \bfO_{i-1}; A_i \right) \gamma^* \left( \left. Y_i \right| X_i, A_i \right)} \right] \\
        = \ & \sum_{i=1}^{n} \bbE_{\left( \Xi_s, \Gamma^* \right)} \left[ \log \frac{\xi_s \left( X_i \right)}{\xi^* \left( X_i \right)} \right] \\
        = \ & n \cdot \textsf{KL} \left( \left. \Xi_s \right\| \Xi^* \right).
    \end{split}
\end{equation}
So if one can show that $\Xi_s (\cdot) \in \calN \left( \Xi^* \right)$, then the equation \eqref{eqn:proof_thm:local_minimax_lower_bound_v1_v4} guarantees that
\[
    \textsf{KL} \left( \left. \bbP_{\left( \Xi_s, \Gamma^* \right)}^{n} \right\| \bbP_{\left( \Xi^*, \Gamma^* \right)}^{n} \right) = n \cdot \textsf{KL} \left( \left. \Xi_s \right\| \Xi^* \right) \leq 1,
\]
which can be taken collectively with the bound \eqref{eqn:proof_thm:local_minimax_lower_bound_v1_v2} to produce the following conclusion:
\begin{equation}
    \label{eqn:proof_thm:local_minimax_lower_bound_v1_v5}
    \begin{split}
        \textsf{TV} \left( \bbP_{\left( \Xi_s, \Gamma^* \right)}^{n}, \bbP_{\left( \Xi^*, \Gamma^* \right)}^{n} \right) \leq \frac{1}{\sqrt{2}}.
    \end{split}
\end{equation}

With the arguments thus far in place, it remains to construct a family $\left\{ \Xi_s (\cdot) \in \Delta (\bbX): s \in \left( 0, +\infty \right) \right\}$ and then choose a parameter $s > 0$ such that $\Xi_s \in \calN \left( \Xi^* \right)$ and the functional values $\tau \left( \Xi_s, \Gamma^* \right)$ and $\tau \left( \Xi^*, \Gamma^* \right)$ are well-separated. To this end, we consider the function $\tilde{h} (\cdot): \bbX \to \bbR$ defined by
\[
    \tilde{h} (x) :=
    \begin{cases}
        h (x) & \textnormal{if } \left| h (x) \right| \leq 2 H_{2 \to 4} \sqrt{\bbE_{X \sim \Xi^*} \left[ h^2 (X) \right]}; \\
        \textsf{sign} \left( h (x) \right) \sqrt{\bbE_{X \sim \Xi^*} \left[ h^2 (X) \right]} & \textnormal{otherwise.}
    \end{cases}
\]
Since $H_{2 \to 4} \geq 1$, one can easily find that $\left| \tilde{h} (x) \right| \leq \left| h (x) \right|$ for all $x \in \bbX$. Now for each $s \in \left( 0, +\infty \right)$, we define the \emph{tilted probability measure} $\Xi_s (\cdot) \in \Delta (\bbX)$ by
\begin{equation}
    \label{eqn:proof_thm:local_minimax_lower_bound_v1_v6}
    \begin{split}
        \xi_s (x) = \frac{\mathrm{d} \Xi_s}{\mathrm{d} \lambda_{\bbX}} (x) := \frac{1}{\calZ (s)} \xi^* (x) \exp \left( s \tilde{h} (x) \right),\ \forall x \in \bbX,
    \end{split}
\end{equation}
where $\calZ (s) := \int_{\bbX} \xi^* (x) \exp \left( s \tilde{h} (x) \right) \mathrm{d} \lambda_{\bbX} (x) = \bbE_{X \sim \Xi^*} \left[ \exp \left( s \tilde{h} (X) \right) \right]$. At this point, we note for every $x \in \bbX$ that
\begin{equation}
    \label{eqn:proof_thm:local_minimax_lower_bound_v1_v7}
    \begin{split}
        \exp \left( - s \left\| \tilde{h} \right\|_{\infty} \right) \leq \exp \left( s \tilde{h} (x) \right) \leq \exp \left( s \left\| \tilde{h} \right\|_{\infty} \right), 
    \end{split}
\end{equation}
which also immediately yields
\begin{equation}
    \label{eqn:proof_thm:local_minimax_lower_bound_v1_v8}
    \begin{split}
        \exp \left( - s \left\| \tilde{h} \right\|_{\infty} \right) \leq \calZ (s) = \bbE_{X \sim \Xi^*} \left[ \exp \left( s \tilde{h} (X) \right) \right] \leq \exp \left( s \left\| \tilde{h} \right\|_{\infty} \right).
    \end{split}
\end{equation}
Here, we choose $s = \frac{1}{4 \left\| h \right\|_{L^2 \left( \Xi^* \right)} \sqrt{n}} > 0$. Then, it holds due to the fact $\left| \tilde{h} (x) \right| \leq 2 H_{2 \to 4} \left\| h \right\|_{L^2 \left( \Xi^* \right)}$ for all $x \in \bbX$ that
\begin{equation}
    \label{eqn:proof_thm:local_minimax_lower_bound_v1_v9}
    \begin{split}
        s \left\| \tilde{h} \right\|_{\infty} = \frac{1}{4 \sqrt{n}} \cdot \frac{\left\| \tilde{h} \right\|_{\infty}}{\left\| h \right\|_{L^2 \left( \Xi^* \right)}} \leq \frac{H_{2 \to 4}}{2 \sqrt{n}} \stackrel{\textnormal{(a)}}{\leq} \frac{1}{8},
    \end{split}
\end{equation}
where the step (a) follows due to the assumption that $n \geq 16 H_{2 \to 4}^2$. Now, it's time to prove that $\Xi_s \in \calN \left( \Xi^* \right)$ for the current choice of the parameter $s > 0$. Due to Theorem 5 in \cite{gibbs2002choosing}, it follows that
\begin{equation}
    \label{eqn:proof_thm:local_minimax_lower_bound_v1_v10}
    \begin{split}
        \textsf{KL} \left( \left. \Xi_s \right\| \Xi^* \right) \leq \log \left\{ 1 + \chi^2 \left( \left. \Xi_s \right\| \Xi^* \right) \right\} \leq \chi^2 \left( \left. \Xi_s \right\| \Xi^* \right).
    \end{split}
\end{equation}
So it suffices to upper bound the $\chi^2$-divergence $\chi^2 \left( \left. \Xi_s \right\| \Xi^* \right)$. One can reveal that
\begin{align}
    \label{eqn:proof_thm:local_minimax_lower_bound_v1_v11}
    \chi^2 \left( \left. \Xi_s \right\| \Xi^* \right) = \ & \textsf{Var}_{X \sim \Xi^*} \left[ \frac{\xi_s (X)}{\xi^* (X)} \right] \nonumber \\
    = \ & \frac{1}{\calZ^2 (s)} \textsf{Var}_{X \sim \Xi^*} \left[ \exp \left( s \tilde{h} (X) \right) \right] \nonumber \\
    \leq \ & \frac{1}{\calZ^2 (s)} \bbE_{X \sim \Xi^*} \left[ \left\{ \exp \left( s \tilde{h} (X) \right) - 1 \right\}^2 \right] \\
    \stackrel{\textnormal{(b)}}{\leq} \ & \exp \left( 2 s \left\| \tilde{h} \right\|_{\infty} \right) \bbE_{X \sim \Xi^*} \left[ \exp \left( 2 s \left| \tilde{h} (X) \right| \right) \cdot s^2 \tilde{h}^2 (X) \right] \nonumber \\
    \stackrel{\textnormal{(c)}}{\leq} \ & \exp \left( 4 s \left\| \tilde{h} \right\|_{\infty} \right) \cdot s^2 \bbE_{X \sim \Xi^*} \left[ h^2 (X) \right], \nonumber
\end{align}
where the step (b) makes use of the fact \eqref{eqn:proof_thm:local_minimax_lower_bound_v1_v8} together with the elementary bound $\left| \exp (u) - 1 \right| \leq |u| \exp \left( |u| \right)$, $\forall u \in \bbR$, and the step (c) follows from the fact $\left| \tilde{h} (x) \right| \leq \left| h (x) \right|$, $\forall x \in \bbX$. If we put $s = \frac{1}{4 \left\| h \right\|_{L^2 \left( \Xi^* \right)} \sqrt{n}}$ into the bound \eqref{eqn:proof_thm:local_minimax_lower_bound_v1_v11}, then we obtain from the fact $s \left\| \tilde{h} \right\|_{\infty} \leq \frac{1}{8}$ together with the basic inequality \eqref{eqn:proof_thm:local_minimax_lower_bound_v1_v10} that
\begin{equation}
    \label{eqn:proof_thm:local_minimax_lower_bound_v1_v12}
    \begin{split}
        \textsf{KL} \left( \left. \Xi_s \right\| \Xi^* \right) \leq \chi^2 \left( \left. \Xi_s \right\| \Xi^* \right) \leq 2 s^2 \left\| h \right\|_{L^2 \left( \Xi^* \right)}^2 = \frac{1}{8n},
    \end{split}
\end{equation}
which implies $\Xi_s \in \calN \left( \Xi^* \right)$ for the choice of the parameter $s = \frac{1}{4 \left\| h \right\|_{L^2 \left( \Xi^* \right)} \sqrt{n}} > 0$. Therefore, the upper bound on the total variation distance \eqref{eqn:proof_thm:local_minimax_lower_bound_v1_v5} turns out to be valid.
\medskip

Next, we lower bound the gap between the functional values $\tau \left( \Xi_s, \Gamma^* \right)$ and $\tau \left( \Xi^*, \Gamma^* \right)$. It holds that
\begin{equation}
    \label{eqn:proof_thm:local_minimax_lower_bound_v1_v13}
    \begin{split}
        \tau \left( \Xi_s, \Gamma^* \right) - \tau \left( \Xi^*, \Gamma^* \right) = \ & \bbE_{X \sim \Xi_s} \left[ \left\langle g (X, \cdot), \mu^* (X, \cdot) \right\rangle_{\lambda_{\bbA}} \right] - \tau \left( \calI^* \right) \\
        = \ & \frac{1}{\calZ (s)} \int_{\bbX} \xi^* (x) \exp \left( s \tilde{h} (x) \right) \underbrace{\left\{ \left\langle g (x, \cdot), \mu^* (x, \cdot) \right\rangle_{\lambda_{\bbA}} - \tau \left( \calI^* \right) \right\}}_{= \ h (x)} \mathrm{d} \lambda_{\bbX} (x) \\
        = \ & \frac{1}{\calZ (s)} \bbE_{X \sim \Xi^*} \left[ h (X) \exp \left( s \tilde{h} (X) \right) \right] \\
        = \ & \frac{\bbE_{X \sim \Xi^*} \left[ h (X) \exp \left( s \tilde{h} (X) \right) \right]}{\bbE_{X \sim \Xi^*} \left[ \exp \left( s \tilde{h} (X) \right) \right]}.
    \end{split}
\end{equation}
Since $s \left\| \tilde{h} \right\|_{\infty} \leq \frac{1}{8}$, we have $s \tilde{h} (X) \in \left[ - \frac{1}{4}, \frac{1}{4} \right]$ and therefore the simple inequality
\[
    \left| \exp (u) - 1 - u \right| \leq u^2,\ \forall u \in \left[ - \frac{1}{4}, \frac{1}{4} \right],
\]
implies
\begin{align}
    \label{eqn:proof_thm:local_minimax_lower_bound_v1_v14}
    &\bbE_{X \sim \Xi^*} \left[ h (X) \exp \left( s \tilde{h} (X) \right) \right] \nonumber \\
    \stackrel{\textnormal{(d)}}{\geq} \ & \underbrace{\bbE_{X \sim \Xi^*} \left[ h (X) \right]}_{= \ 0} + s \bbE_{X \sim \Xi^*} \left[ \left| h (X) \right| \left| \tilde{h} (X) \right| \right] - s^2 \bbE_{X \sim \Xi^*} \left[ \left| h (X) \right| \tilde{h}^2 (X) \right] \nonumber \\
    \stackrel{\textnormal{(e)}}{\geq} \ & s \bbE_{X \sim \Xi^*} \left[ \tilde{h}^2 (X) \right] - s^2 \sqrt{\bbE_{X \sim \Xi^*} \left[ h^2 (X) \right]} \underbrace{\sqrt{\bbE_{X \sim \Xi^*} \left[ h^4 (X) \right]}}_{= \ H_{2 \to 4} \cdot \bbE_{X \sim \Xi^*} \left[ h^2 (X) \right]} \\
    \stackrel{\textnormal{(f)}}{\geq} \ & \frac{s}{2} \bbE_{X \sim \Xi^*} \left[ h^2 (X) \right] - s^2 H_{2 \to 4} \left( \bbE_{X \sim \Xi^*} \left[ h^2 (X) \right] \right)^{\frac{3}{2}} \nonumber \\
    = \ & \frac{\left\| h \right\|_{L^2 \left( \Xi^* \right)}}{8} \left( \frac{1}{\sqrt{n}} - \frac{H_{2 \to 4}}{2n} \right) \nonumber \\
    \stackrel{\textnormal{(g)}}{\geq} \ & \frac{\left\| h \right\|_{L^2 \left( \Xi^* \right)}}{16 \sqrt{n}}, \nonumber
\end{align}
where the step (d) follows due to the fact that $\textsf{sign} \left( h (x) \right) = \textsf{sign} \left( \tilde{h} (x) \right)$, $\forall x \in \bbX$, the step (e) makes use of the property that $\left| \tilde{h} (x) \right| \leq \left| h (x) \right|$, $\forall x \in \bbX$, together with the Cauchy-Schwarz inequality, the step (f) follows due to Lemma 7 in \cite{mou2022off}, and the step (g) utilizes the assumption that $n \geq 16 H_{2 \to 4}^2$. Putting the lower bound \eqref{eqn:proof_thm:local_minimax_lower_bound_v1_v14} into the equation \eqref{eqn:proof_thm:local_minimax_lower_bound_v1_v13} yields
\begin{equation}
    \label{eqn:proof_thm:local_minimax_lower_bound_v1_v15}
    \begin{split}
        \tau \left( \Xi_s, \Gamma^* \right) - \tau \left( \Xi^*, \Gamma^* \right) \geq  \frac{\left\| h \right\|_{L^2 \left( \Xi^* \right)}}{16 \sqrt{n} \bbE_{X \sim \Xi^*} \left[ \exp \left( s \tilde{h} (X) \right) \right]} \stackrel{\textnormal{(h)}}{\geq} \frac{\left\| h \right\|_{L^2 \left( \Xi^* \right)}}{24 \sqrt{n}},
    \end{split}
\end{equation}
where the step (h) holds since $\bbE_{X \sim \Xi^*} \left[ \exp \left( s \tilde{h} (X) \right) \right] \leq \frac{3}{2}$, which follows from the fact $\left| s \tilde{h} (X) \right| \leq \frac{1}{8}$. Finally, by taking three pieces \eqref{eqn:proof_thm:local_minimax_lower_bound_v1_v1}, \eqref{eqn:proof_thm:local_minimax_lower_bound_v1_v5}, and \eqref{eqn:proof_thm:local_minimax_lower_bound_v1_v15} collectively, one completes the proof of Lemma \ref{lemma:local_minimax_lower_bound_v1}.

\subsection{Proof of Lemma \ref{lemma:local_minimax_lower_bound_v2}}
\label{subsec:proof_lemma:local_minimax_lower_bound_v2}

The proof of Lemma \ref{lemma:local_minimax_lower_bound_v2} is also heavily relies on Le Cam's two-point method. For $(i, s, z) \in [n] \times \left( 0, +\infty \right) \times \left\{ \pm 1 \right\}$, we consider the function $\mu_i (zs) (\cdot, \cdot): \bbX \times \bbA \to \bbR$ defined by
\begin{equation}
    \label{eqn:proof_thm:local_minimax_lower_bound_v2_v1}
    \begin{split}
        \mu_i (zs) (x, a) := \mu^* (x, a) + \frac{zs g(x, a)}{\overline{\pi}_i (x, a)} \sigma^2 (x, a),\ \forall (x, a) \in \bbX \times \bbA.
    \end{split}
\end{equation}
Also, we define the perturbed outcome kernel $\Gamma_i (zs) (\cdot, \cdot): \bbX \times \bbA \to \bbY$ as
\[
    \Gamma_i (zs) \left( \left. \cdot \right| x, a \right) := \calN \left( \mu_i (zs) (x, a), \sigma^2 (x, a) \right),\ \forall (x, a) \in \bbX \times \bbA.
\]
Then, due to Le Cam's two-point lemma, the local minimax risk over the class $\calC_{\delta} \left( \calI^* \right)$ can be lower bounded by
\begin{equation}
    \label{eqn:proof_thm:local_minimax_lower_bound_v2_v2}
    \begin{split}
        \calM_n \left( \calC_{\delta} \left( \calI^* \right) \right) \geq \frac{1}{4} \left\{ 1 - \textsf{TV} \left( \bbP_{\left( \Xi^*, \Gamma_i (s) \right)}^{n}, \bbP_{\left( \Xi^*, \Gamma_i (-s) \right)}^{n} \right) \right\} \left\{ \tau \left( \Xi^*, \Gamma_i (s) \right) - \tau \left( \Xi^*, \Gamma_i (-s) \right) \right\}^2,
    \end{split}
\end{equation}
provided that $\Gamma_i (zs) \in \calN_{\delta} \left( \Gamma^* \right)$ for $z \in \left\{ \pm 1 \right\}$.
\medskip

\indent We first derive an upper bound on the total variation distance $\textsf{TV} \left( \bbP_{\left( \Xi^*, \Gamma_i (s) \right)}^{n}, \bbP_{\left( \Xi^*, \Gamma_i (-s) \right)}^{n} \right)$. The Pinsker-Csisz\'{a}r-Kullback inequality implies that
\begin{equation}
    \label{eqn:proof_thm:local_minimax_lower_bound_v2_v3}
    \begin{split}
        \textsf{TV} \left( \bbP_{\left( \Xi^*, \Gamma_i (s) \right)}^{n}, \bbP_{\left( \Xi^*, \Gamma_i (-s) \right)}^{n} \right) \leq \sqrt{\frac{1}{2}
        \textsf{KL} \left( \left. \bbP_{\left( \Xi^*, \Gamma_i (s) \right)}^{n} \right\| \bbP_{\left( \Xi^*, \Gamma_i (-s) \right)}^{n} \right)}.
    \end{split}
\end{equation}
The \textsf{KL}-divergence $\textsf{KL} \left( \left. \bbP_{\left( \Xi^*, \Gamma_i (s) \right)}^{n} \right\| \bbP_{\left( \Xi^*, \Gamma_i (-s) \right)}^{n} \right)$ can be computed as
\begin{equation}
    \label{eqn:proof_thm:local_minimax_lower_bound_v2_v4}
    \begin{split}
        &\textsf{KL} \left( \left. \bbP_{\left( \Xi^*, \Gamma_i (s) \right)}^{n} \right\| \bbP_{\left( \Xi^*, \Gamma_i (-s) \right)}^{n} \right) \\
        = \ & \bbE_{\left( \Xi^*, \Gamma_i (s) \right)} \left[ \log \frac{p_{\left( \Xi^*, \Gamma_i (s) \right)}^n \left( \bfO_n \right)}{p_{\left( \Xi^*, \Gamma_i (-s) \right)}^n \left( \bfO_n \right)} \right] \\
        = \ & \bbE_{\left( \Xi^*, \Gamma_i (s) \right)} \left[ \sum_{i=1}^{n} \log \frac{\xi^* \left( X_i \right) \pi_{i}^* \left( X_i, \bfO_{i-1}; A_i \right) \gamma_i (s) \left( \left. Y_i \right| X_i, A_i \right)}{\xi^* \left( X_i \right) \pi_{i}^* \left( X_i, \bfO_{i-1}; A_i \right) \gamma_i (-s) \left( \left. Y_i \right| X_i, A_i \right)} \right] \\
        = \ & \sum_{i=1}^{n} \bbE_{\left( \Xi^*, \Gamma_i (s) \right)} \left[ \log \frac{\gamma_i (s) \left( \left. Y_i \right| X_i, A_i \right)}{\gamma_i (-s) \left( \left. Y_i \right| X_i, A_i \right)} \right].
    \end{split}
\end{equation}
Note that
\begin{equation}
    \label{eqn:proof_thm:local_minimax_lower_bound_v2_v5}
    \begin{split}
        \log \frac{\gamma_i (s) \left( \left. y \right| x, a \right)}{\gamma_i (- s) \left( \left. y \right| x, a \right)} = \ & - \frac{1}{2 \sigma^2 (x, a)} \left[ \left\{ y - \mu_i (s) (x, a) \right\}^2 - \left\{ y - \mu_i (-s) (x, a) \right\}^2 \right] \\
        = \ & \frac{s g (x, a)}{\overline{\pi}_i (x, a)} \left\{ 2 y - \mu_i (s) (x, a) - \mu_i (-s) (x, a) \right\}.
    \end{split}
\end{equation}
By utilizing the fact \eqref{eqn:proof_thm:local_minimax_lower_bound_v2_v5}, one can obtain from the equation \eqref{eqn:proof_thm:local_minimax_lower_bound_v2_v4} that
\begin{align}
    \label{eqn:proof_thm:local_minimax_lower_bound_v2_v6}
    &\textsf{KL} \left( \left. \bbP_{\left( \Xi^*, \Gamma_i (s) \right)}^{n} \right\| \bbP_{\left( \Xi^*, \Gamma_i (-s) \right)}^{n} \right) \nonumber \\
    = \ & \sum_{i=1}^{n} \bbE_{\left( \Xi^*, \Gamma_i (s) \right)} \left[ \bbE_{\left( \Xi^*, \Gamma_i (s) \right)} \left[ \left. \frac{s g \left( X_i, A_i \right)}{\overline{\pi}_i \left( X_i, A_i \right)} \left\{ 2 Y_i - \mu_i (s) \left( X_i, A_i \right) - \mu_i (-s) \left( X_i, A_i \right) \right\} \right| \left( X_i, A_i, \calH_{i-1} \right) \right] \right] \nonumber \\
    = \ & \sum_{i=1}^{n} \bbE_{\left( \Xi^*, \Gamma_i (s) \right)} \left[ \frac{s g \left( X_i, A_i \right)}{\overline{\pi}_i \left( X_i, A_i \right)} \left\{ \mu_i (s) \left( X_i, A_i \right) - \mu_i (-s) \left( X_i, A_i \right) \right\} \right] \nonumber \\
    = \ & 2 s^2 \sum_{i=1}^{n} \bbE_{\left( \Xi^*, \Gamma_i (s) \right)} \left[ \frac{g^2 \left( X_i, A_i \right) \sigma^2 \left( X_i, A_i \right)}{\overline{\pi}_{i}^2 \left( X_i, A_i \right)} \right] \\
    = \ & 2 s^2 \sum_{i=1}^{n} \bbE_{\calI^*} \left[ \frac{g^2 \left( X_i, A_i \right) \sigma^2 \left( X_i, A_i \right)}{\overline{\pi}_{i}^2 \left( X_i, A_i \right)} \right] \nonumber \\
    \stackrel{\textnormal{(a)}}{\leq} \ & 2 K^2 s^2 \sum_{i=1}^{n} \bbE_{\calI^*} \left[ \frac{g^2 \left( X_i, A_i \right) \sigma^2 \left( X_i, A_i \right)}{\left( \pi_{i}^* \right)^2 \left( X_i, \bfO_{i-1}; A_i \right)} \right] \nonumber \\
    = \ & 2 K^2 s^2 n \left\| \sigma \right\|_{(n)}^2, \nonumber
\end{align}
where the step (a) follows from the assumption \eqref{eqn:assumption_reference_Markov_policies}. If we put $s = \frac{1}{2K \sqrt{n} \left\| \sigma \right\|_{(n)}}$ into the inequality \eqref{eqn:proof_thm:local_minimax_lower_bound_v2_v6}, it follows that $\textsf{KL} \left( \left. \bbP_{\left( \Xi^*, \Gamma_i (s) \right)}^{n} \right\| \bbP_{\left( \Xi^*, \Gamma_i (-s) \right)}^{n} \right) \leq \frac{1}{2}$. Hence, by combining this conclusion together with the basic inequality \eqref{eqn:proof_thm:local_minimax_lower_bound_v2_v3}, we arrive at
\begin{equation}
    \label{eqn:proof_thm:local_minimax_lower_bound_v2_v7}
    \begin{split}
        \textsf{TV} \left( \bbP_{\left( \Xi^*, \Gamma_i (s) \right)}^{n}, \bbP_{\left( \Xi^*, \Gamma_i (-s) \right)}^{n} \right) \leq \frac{1}{2}.
    \end{split}
\end{equation}
At this point, we should note for every $(i, z, x, a) \in [n] \times \left\{ \pm 1 \right\} \times \bbX \times \bbA$ that
\begin{equation}
    \label{eqn:proof_thm:local_minimax_lower_bound_v2_v8}
    \begin{split}
        \left| \mu^* (x, a) - \mu_i (sz) (x, a) \right| = \ & \frac{s \left| g (x, a) \right| \sigma^2 (x, a)}{\overline{\pi}_i (x, a)} \\
        = \ & \frac{1}{2 \sqrt{K}} \cdot \frac{\left| g (x, a) \right| \sigma^2 (x, a)}{\sqrt{n} \overline{\pi}_i (x, a) \left\| \sigma \right\|_{(n)}} \\
        \stackrel{\textnormal{(b)}}{\leq} \ & \frac{\delta (x, a)}{2 \sqrt{K}} \\
        \stackrel{\textnormal{(c)}}{\leq} \ & \delta (x, a),
    \end{split}
\end{equation}
where the step (b) holds due to Assumption \ref{assumption:ln}, and the step (c) utilizes the fact that $K \geq 1$, which establishes that $\Gamma_i (zs) \in \calN_{\delta} \left( \Gamma^* \right)$ for $z \in \left\{ \pm 1 \right\}$ and thus the local minimax lower bound \eqref{eqn:proof_thm:local_minimax_lower_bound_v2_v2} turns out to be valid.
\medskip

\indent Next, it's time to aim at establishing a lower bound on the gap between the functional values $\tau \left( \Xi^*, \Gamma_i (s) \right)$ and $\tau \left( \Xi^*, \Gamma_i (-s) \right)$. One can observe that
\begin{equation}
    \begin{split}
        \label{eqn:proof_thm:local_minimax_lower_bound_v2_v9}
        \tau \left( \Xi^*, \Gamma_i (s) \right) - \tau \left( \Xi^*, \Gamma_i (-s) \right) = \ & \bbE_{X \sim \Xi^*} \left[ \left\langle g (X, \cdot), \mu_i (s) (X, \cdot) - \mu_i (-s) (X, \cdot) \right\rangle_{\lambda_{\bbA}} \right] \\
        = \ & 2s \cdot \bbE_{\calI^*} \left[ \int_{\bbA} \frac{g^2 \left( X_i, a \right) \sigma^2 \left( X_i, a \right)}{\overline{\pi}_i \left( X_i, a \right)} \mathrm{d} \lambda_{\bbA} (a) \right] \\
        \stackrel{\textnormal{(d)}}{\geq} \ & \frac{2s}{K} \cdot \bbE_{\calI^*} \left[ \int_{\bbA} \frac{g^2 \left( X_i, a \right) \sigma^2 \left( X_i, a \right)}{\pi_{i}^* \left( X_i, \bfO_{i-1}; a \right)} \mathrm{d} \lambda_{\bbA} (a) \right] \\
        = \ & \frac{2s}{K} \cdot \bbE_{\calI^*} \left[ \frac{g^2 \left( X_i, A_i \right) \sigma^2 \left( X_i, A_i \right)}{\left( \pi_{i}^* \right)^2 \left( X_i, \bfO_{i-1}; A_i \right)} \right] \\
        = \ & \frac{1}{K^{2} \sqrt{n} \left\| \sigma \right\|_{(n)}} \bbE_{\calI^*} \left[ \frac{g^2 \left( X_i, A_i \right) \sigma^2 \left( X_i, A_i \right)}{\left( \pi_{i}^* \right)^2 \left( X_i, \bfO_{i-1}; A_i \right)} \right],
    \end{split}
\end{equation}
where the step (d) holds by the assumption \eqref{eqn:assumption_reference_Markov_policies}. Taking three pieces \eqref{eqn:proof_thm:local_minimax_lower_bound_v2_v2}, \eqref{eqn:proof_thm:local_minimax_lower_bound_v2_v7}, and \eqref{eqn:proof_thm:local_minimax_lower_bound_v2_v9} collectively, we have
\begin{equation}
    \label{eqn:proof_thm:local_minimax_lower_bound_v2_v10}
    \begin{split}
        \calM_n \left( \calC_{\delta} \left( \calI^* \right) \right) \geq \frac{1}{8 K^4 n \left\| \sigma \right\|_{(n)}^2} \left( \bbE_{\calI^*} \left[ \frac{g^2 \left( X_i, A_i \right) \sigma^2 \left( X_i, A_i \right)}{\left( \pi_{i}^* \right)^2 \left( X_i, \bfO_{i-1}; A_i \right)} \right] \right)^2
    \end{split}
\end{equation}
for all $i \in [n]$. Hence, one can conclude by taking an average of the local minimax lower bound \eqref{eqn:proof_thm:local_minimax_lower_bound_v2_v10} over $i \in [n]$ that
\begin{equation}
    \label{eqn:proof_thm:local_minimax_lower_bound_v2_v11}
    \begin{split}
        \calM_n \left( \calC_{\delta} \left( \calI^* \right) \right) = \ & \frac{1}{n} \sum_{i=1}^{n} \calM_n \left( \calC_{\delta} \left( \calI^* \right) \right) \\
        \geq \ & \frac{1}{8 K^4 n^2 \left\| \sigma \right\|_{(n)}^2} \sum_{i=1}^{n} \left( \bbE_{\calI^*} \left[ \frac{g^2 \left( X_i, A_i \right) \sigma^2 \left( X_i, A_i \right)}{\left( \pi_{i}^* \right)^2 \left( X_i, \bfO_{i-1}; A_i \right)} \right] \right)^2 \\
        \stackrel{\textnormal{(e)}}{\geq} \ & \frac{1}{8 K^4 n^3 \left\| \sigma \right\|_{(n)}^2} \underbrace{\left( \sum_{i=1}^{n} \bbE_{\calI^*} \left[ \frac{g^2 \left( X_i, A_i \right) \sigma^2 \left( X_i, A_i \right)}{\left( \pi_{i}^* \right)^2 \left( X_i, \bfO_{i-1}; A_i \right)} \right] \right)^2}_{= \ n^2 \left\| \sigma \right\|_{(n)}^4} \\
        = \ & \frac{1}{8 K^4} \cdot \frac{\left\| \sigma \right\|_{(n)}^2}{n},
    \end{split}
\end{equation}
where the step (e) makes use of the Cauchy-Schwarz inequality.

\end{document}